\documentclass{WileyMSP-template}

\usepackage{tikz}
\usetikzlibrary{shapes, shapes.geometric, arrows, calc}
\usetikzlibrary{bayesnet}

\usepackage[ruled,vlined]{algorithm2e}

\usepackage{microtype}
\usepackage{graphicx}
\usepackage[font=small]{caption}
\usepackage{subfigure}

\usepackage{url}            % simple URL typesetting
\usepackage{booktabs}       % professional-quality tables
\usepackage{tabularx}
\usepackage{amsfonts}       % blackboard math symbols
\usepackage{nicefrac}       % compact symbols for 1/2, etc.
\usepackage{wrapfig}

\usepackage{amsmath,amsfonts,bm}
\usepackage{amssymb}

\usepackage{amsthm}
\usepackage{mathtools}
\usepackage{mathrsfs}

\usepackage{pifont}
\usepackage{graphics} % for pdf, bitmapped graphics files
\usepackage{svg}
\usepackage{tikz}
\usepackage{pgfkeys}
\usepackage{standalone}
\usepackage{cite}
\usepackage{graphicx}
\usepackage{multirow}
\usepackage{threeparttable}
\usepackage{placeins}
\usepackage{adjustbox}
\usepackage[super]{nth}
\usepackage{sidecap}

\usepackage{hyperref}
\usepackage{siunitx}
\usepackage{cases}

\usepackage{graphicx}
\usepackage{sidecap} % Added package for side captions
\usepackage{caption} % Added package for caption formatting

\newtheorem{proposition}{Proposition}
\newtheorem{corollary}{Corollary}
\newtheorem{remark}{Remark}

\usepackage{color,colortbl}
\definecolor{orange}{rgb}{1,0.647,0}
\definecolor{darkgreen}{rgb}{0,0.5,0}

\usepackage{float}

\usepackage{balance} 

\begin{document}

%\pagestyle{fancy}
%\rhead{\includegraphics[width=2.5cm]{vch-logo.png}}

% \title{Physics-Inspired Neural Networks: An Exploration on Robotics}
% \title{Using Physics-informed Neural Networks to Model and Control Robots: a Theoretical and Experimental Investigation}
\title{Physics-informed Neural Networks to Model and Control Robots: a Theoretical and Experimental Investigation}

\maketitle

% Author: Please give full first and last names for authors and include * after the name of all corresponding authors

\author{Jingyue Liu$^{1,*}$}
\author{Pablo Borja$^{2}$}
\author{Cosimo Della Santina$^{1,3}$}

% Dedication

%\dedication{Optional dedication here. If no dedication is required, please leave blank}

% Affiliations: Please provide adacemic titles (Prof. or Dr.) for all authors where applicable, and include an institutional email address for all corresponding authors
\begin{affiliations}
$^{1}$Department of Cognitive Robotics, Delft University of Technology, Delft 2628 CD, The Netherlands \\
\texttt{\{J.Liu-14, C.DellaSantina\}@tudelft.nl}\\[0.1cm]
$^{2}$School of Engineering, Computing and Mathematics, University of Plymouth, Plymouth PL4 8AA,\\ United Kingdom.
\texttt{pablo.borjarosales@plymouth.ac.uk}
\\[0.1cm]
$^{3}$Institute of Robotics and Mechatronics German Aerospace Center (DLR) Oberpfaffenhofen 82234,\\ Germany\\[0.1cm]
$^*$Corresponding author
\end{affiliations}

% Keywords: Please provide a minimum of three and a maximum of seven keywords, separated by commas


% Abstract should be written in the present tense and impersonal style (i.e., avoid we) and be at most 200 words long
\begin{abstract} \normalsize
This work concerns the application of physics-informed neural networks to the modeling and control of complex robotic systems. Achieving this goal required extending Physics Informed Neural Networks to handle non-conservative effects. We propose to combine these learned models with model-based controllers originally developed with first-principle models in mind. By combining standard and new techniques, we can achieve precise control performance while proving theoretical stability bounds. These validations include real-world experiments of motion prediction with a soft robot and of trajectory tracking with a Franka Emika manipulator.
%
% This work concerns the application of Physics-informed neural networks to modeling and model-based control of robotic systems. The contribution of this work is two-fold. From a methodological standpoint, we extend Hamiltonian and Lagrangian Neural Networks, including learning non-conservative effects. Additionally, we introduce a modified loss function that relies on predictions of the model rather than on higher-order derivatives of the measurements. We then discuss the combination of learned models with model-based controllers originally developed with first-principle models in mind. We test the effectiveness of the proposed strategy via simulations and, most importantly, with experiments on soft and rigid robots - including trajectory tracking of a Franka Emika manipulator.
%
\end{abstract}

\section{Introduction}
%
%Deep learning architectures can model nonlinear functions from complex systems through a vast number of parameters. The choice of architecture for deep learning models is problem-specific. Specifically, recurrent neural networks (RNNs) are adept at modeling time-series data, whereas convolutional neural networks (CNNs) exhibit efficacy in modeling image-based data. However, despite the impressive performance of these models, a physically meaningful deep-learning architecture is yet to be established for modeling robotics.

% The successes of Deep Learning (DL) are remarkable and span many fields, including robotics. Here, DL has proven its potential in dealing with tasks such as vision-guided navigation \cite{chen2020deep}, grasp-planning \cite{ichnowski2020deep}, human-robot interaction \cite{mukherjee2022survey}, and design \cite{stella2023how} - to name a few. Yet, its application to the generation of motor intelligence in physical systems remains limited. Applications in simulation have shown that Deep Reinforcement Learning can potentially outperform classic approaches \cite{bucsoniu2018reinforcement,wurman2022outracing,rudin2022learning}. Validation in physical applications has been so far mostly constrained by the need for pre-training in simulation \cite{akkaya2019solving,zhao2020sim,kulkarni2022learning}. 
%
Deep Learning (DL) has made significant strides across various fields, with robotics being a salient example. DL has excelled in tasks such as vision-guided navigation \cite{chen2020deep}, grasp-planning \cite{ichnowski2020deep}, human-robot interaction \cite{mukherjee2022survey}, and even design \cite{stella2023how}. Despite this, the application of DL to generate motor intelligence in physical systems remains limited. Deep Reinforcement Learning, in particular, has shown the potential to outperform traditional approaches in simulations \cite{bucsoniu2018reinforcement,wurman2022outracing,rudin2022learning}. However, its transfer to physical applications has been primarily hampered by the prerequisite of pre-training in a simulated environment \cite{akkaya2019solving,zhao2020sim,kulkarni2022learning}.

The central drawback of general-purpose DL lies in its sample inefficiency, stemming from the need to distill all aspects of a task from data \cite{sunderhauf2018limits,antonelli2023data}. In response to these challenges, there's a rising trend in robotics to specifically incorporate geometric priors into data-driven methods to optimize the learning efficiency \cite{beik2021learning,simeonov2022neural,urain2022se}. This approach proves especially advantageous for high-level tasks that need not engage with the system's physics.
%
%Beyond purely geometric priors, physics-inspired neural networks \cite{daw2017physics,karniadakis2021physics} have been proposed outside robotics to infuse the structural knowledge derived from fundamental physics into the architecture and training procedures. These strategies have been successfully applied in various fields, including earth science \cite{chen2021physics}, power systems \cite{huang2022applications}, fluid mechanics \cite{mao2020physics}, and materials science \cite{niaki2021physics}.  In this context, there have been efforts to integrate Lagrangian or Hamiltonian mechanics with deep learning, resulting in the development of models such as Deep Lagrangian Networks (DeLaNs) \cite{lutter2019deep}, Lagrangian Neural Networks (LNNs) \cite{cranmer2020lagrangian}, and Hamiltonian Neural Networks (HNN) \cite{greydanus2019hamiltonian}.  These studies have shown successful application - mostly in simulation - on simple physical systems like the double pendulum and cart pole \cite{roehrl2020modeling,zhong2021benchmarking,bhattoo2022learning,djeumou2022neural}, yielding encouraging results. Yet, their efficacy in modeling complex robotic structures, including high degree-of-freedom rigid and soft robots, remains unexplored - especially when experimental data are concerned. 
%
Physics-inspired neural networks \cite{daw2017physics,karniadakis2021physics,djeumou2022neural}, infusing fundamental physics knowledge into their architecture and training, have found success in various fields outside robotics, from earth science to materials science \cite{chen2021physics,huang2022applications,mao2020physics,niaki2021physics}. In robotics, integration of Lagrangian or Hamiltonian mechanics with deep learning has yielded models like Deep Lagrangian Neural Networks (LNNs) \cite{cranmer2020lagrangian}, and Hamiltonian Neural Networks (HNN) \cite{greydanus2019hamiltonian}. Several extensions have been proposed in the literature, for example, including contact models \cite{zhong2021extending}, or proposing graph formulations \cite{bhattoo2023learning}. 
The potential of Lagrangian and Hamiltonian Neural Networks in learning the dynamics of basic physical systems has been demonstrated in various studies \cite{roehrl2020modeling,zhong2021benchmarking,bhattoo2022learning,djeumou2022neural}. However, the exploration of these techniques in modeling intricate robotic structures, especially with real-world data, is still in its early stages. Notably, \cite{lutter2023combining} applied these methods to a position-controlled robot with four degrees of freedom, which represents a relatively less complex system in comparison to contemporary manipulators.
%These techniques have consistently shown promising results in learning the dynamics of simple physical systems \cite{roehrl2020modeling,zhong2021benchmarking,bhattoo2022learning,djeumou2022neural}. However, their effectiveness in modeling complex robotic structures, particularly using experimental data, remains largely untested. An exception is \cite{lutter2023combining}, which however focuses on a position-controlled robot with only four degrees of freedom - so, a sensibly simpler system compared to modern manipulators.

This work deals with the experimental application of PINN to rigid and soft continuum robots \cite{della2021soft}. Such endeavor required modifying LNN and HNN to fix three issues that prevented their application to these systems: (i) the lack of energy dissipation mechanism, (ii) the assumption that control actions are collocated on the measured configurations, (iii) the need for direct acceleration measurements, which are non-causal and require numerical differentiation. For issue (iii), we borrow a strategy proposed in \cite{gupta2019general,gupta2020structured}, which relies on forward integrating the dynamics, while for (i) and (ii), we propose innovative solutions.

Furthermore, we exploit a central advantage of LNNs and HNNs compared to other learning techniques; the fact that the learned model has the mathematical structure that is usually assumed in robots and mechanical systems control. By forcing such a representation, we use model-based strategies originally developed for first principle models \cite{murray1994mathematical,khalil2015nonlinear,della2023model} to obtain provably stable performance with guarantees of robustness.

The use of PINNs in control has only recently started to be explored. Recent investigations \cite{gupta2020structured,zheng2023physics,sanyal2022ramp} focused on combining PINNs with model predictive control (MPC), thus not exploiting the mathematical structure of the learned equations. Indeed, this strategy is part of an increasingly established trend seeking the combination of (non-PI and non-deep) learned models with MPC \cite{hewing2019cautious,mitsioni2023safe}. Applications to control PDEs are discussed in \cite{arnold2021state,mowlavi2023optimal}, while an application to robotics is investigated in simulation in \cite{nicodemus2022physics}.
% with no applications to robotic systems
%
Preliminary investigations in other model-based techniques are provided in \cite{lutter2019bdeep,lutter2023combining}, where, however, controllers are provided without any guarantee of stability or robustness and formulated for specific cases.

% \subsubsection{First Sub Subsection}

% \threesubsection{First lowest-level subsection}
%\subsection{Contributions}
% In this work, we contribute to state of the art in PINNs and robotics with the following:
% \begin{itemize}
%     \item An approach to include linear damping and the actuators-robot interaction in Lagrangian and Hamiltonian neural networks. This extension enlarges the network's applicability to a broader range of systems.
%     \item Using the Runge-Kutta method to avoid the inaccuracies and environmental susceptibility associated with direct acceleration measurements.
%     \item Assessing the effectiveness of the proposed method through simulations and experiments on rigid and soft robotic systems.
%     \item Designing model-based controllers using the learned models. These more accurate models help the controller reduce steady-state errors and achieve a faster response.
% \end{itemize}
%
To summarize, in this work, we contribute to state of art in PINNs and robotics with the following:
\begin{enumerate}
    \item An approach to include dissipation and allow for non-collocated control actions in Lagrangian and Hamiltonian neural networks, solving issues (i) and (ii).
    % \item A modified loss function that avoids using non-measurable high-order derivatives by integrating the learned dynamics, solving issue (iii).
    \item Controllers for regulation and tracking, grounded in classic nonlinear control that exploit the mathematical structure of the learned models. For the first time, we prove the stability and robustness of these strategies.
    \item Simulations and experiments on articulated and soft continuum robotic systems. %To the Authors' best knowledge, the only other example of the application of PINN to complex robotic systems is \cite{lutter2023combining}, which however focuses on a conventional manipulator only. 
    To the Authors' best knowledge, these are the first validation of PINN, and PINN-based control applied to complex mechanical systems. 
\end{enumerate}

% The rest of the paper is structured as follows: Section \ref{sec: preliminaries} provides the preliminary material, including a brief overview of Lagrangian and Hamiltonian dynamics and deep learning. The modified network structure is detailed in Section \ref{sec: Methods}, along with the experimental setup. Section \ref{sec: Simulation Results} illustrates the predictive capability of the learned models via the simulations of the Franka Emika Panda cobot and two 3D soft manipulators with one and two segments, respectively. Moreover, the performance of the learned model-based controller is assessed in this section. Section \ref{sec: Experimental Validation} presents the experimental validation of the methodology in the Franka Emika Panda robot and a tendon-driven soft robot. Finally, Section \ref{sec:6_Conclusion} is devoted to the concluding remarks.

%
%
%
\section{Preliminaries} \label{sec: preliminaries}
\subsection{Lagrangian and Hamiltonian Dynamics}\label{subsec:LH1}

Robots' dynamics can be represented using Lagrangian or Hamiltonian mechanics. In the former, the state is defined by the generalized coordinates ${q} \in \mathbb{R}^N$ and their velocities $\dot{{q}} \in \mathbb{R}^N$, where $N$ represents the configuration space dimension. The Euler-Lagrange equation dictates the system's behavior $\frac{\mathrm{d} }{\mathrm{d} t} \left(\frac{\partial L({q}, \dot{{q}})}{\partial \dot{{q}}} \right) - \frac{\partial L({q}, \dot{{q}})}{\partial{{q}}} = {F}_{\tt ext}$, where $L({q}, \dot{{q}}) = T({q}, \dot{{q}}) - V({q})$ with potential energy $V$ and kinetic energy $T\frac{1}{2}\dot{{q}}^{\top}{M}({q})\dot{{q}}$, where ${M}({q}) \in \mathbb{R}^{N\times N}$ is the positive definite mass inertia matrix. External forces, denoted as ${F}_{\tt ext}\in\mathbb{R}^N$, include control inputs and dissipation forces.

In Hamiltonian mechanics, momenta ${p} \in \mathbb{R}^N$ replace the velocities, with $\dot{{q}} ={M}^{-1}({q}){p}$. The Hamiltonian equations $  \dot{{q}} = \frac{\partial H({q}, {p})}{\partial {p}}, \quad   \dot{{p}} = -\frac{\partial H({q}, {p})}{\partial {q}}+{F}_{\tt ext},$ where $H({q}, {p}) = T({q}, {p}) + V({q})$ is the total energy. The kinetic energy in this case is defined as $T({q}, {p}) = \frac{1}{2}{p}^{\top}{M}^{-1}({q}){p}$.

\subsection{LNNs and HNNs}\label{subsec:LNNHNN1}

Lagrangian Neural Networks (LNNs) employ the principle of least action to learn a Lagrangian function $\mathcal{L}(q,\dot{q})$ from trajectory data, with the learned function generating dynamics via standard Euler-Lagrange machinery \cite{murray1994mathematical}. The loss function for the LNN is given by the Mean Squared Error between the actual accelerations $\ddot{q}$ and the ones that the learned model would expect $\hat{\ddot{q}}$
\begin{equation}\label{eq:loss_LNN_vanilla}
    \mathcal{L}_{\mathrm{LNN}} = \text{MSE}(\ddot{q}, \hat{\ddot{q}}).
\end{equation}
HNNs, conversely, are designed to learn the Hamiltonian function $H(p, q)$. Once learned, this Hamiltonian function provides dynamics through Hamilton's equations. The loss function for HNN is similarly an MSE but between the predicted and actual time derivatives of generalized coordinates and momenta:
\begin{equation}\label{eq:loss_HNN_vanilla}
\mathcal{L}_{\mathrm{HNN}} = \text{MSE}((\dot{q}, \dot{p}), (\hat{\dot{q}}, \hat{\dot{p}})).
\end{equation}
We use fully connected neural networks with multiple layers of neurons with associated weights to learn the Lagrangian or the Hamiltonian, shown in Figure \ref{fig:FNN}. % The learning purpose is to establish a mapping function, $f_{\tt NN}$, from the input data $\mathcal{D}$ to the output labels $\mathcal{T}$, i.e., $f_{\tt NN}: \mathcal{D} \to \mathcal{T}$. 

\subsection{Limits of classic LNNs and HNNs}

Note that both loss functions rely on measuring derivatives of the state $\ddot{q}$ and $\dot{p}$, which - by definition of state - cannot be directly measured. This issue is easily circumvented in simulation by the use of a non-causal sensor. Yet, this is not a feasible solution with physical experiments. An unrobust alternative is to estimate these values from measurements of positions and velocities numerically. This relates to issue (iii), stated in the introduction.

Moreover, existing LNNs and HNNs assume that ${F}_{\tt ext}\in\mathbb{R}^N$ is directly measured. This is a reasonable hypothesis only if the system is conservative, fully actuated, and the actuation is collocated. The first characteristic is never fulfilled by real systems, while the second and the third are very restrictive outside when dealing with innovative robotic solutions as soft \cite{della2021soft} or flexible robots \cite{della2021flexible}. Note that learning-based control is imposing itself as a central trend in these non-conventional robotic systems \cite{laschi2023learning}. These considerations relate to issues (i) and (ii) stated in the introduction.

% fully connected network is favored due to its capability of non-linear regression, which can be achieved through activation functions---such as sigmoid or ReLU---after the weighted combination of the input data. fully connected network serves as a foundational and efficient structure in deep learning, inspiring the development of other advanced structures such as GNNs, LNNs, DeLaNs, and HNNs.

\begin{figure}[h]
    \centering
    \includegraphics[width=0.3\linewidth]{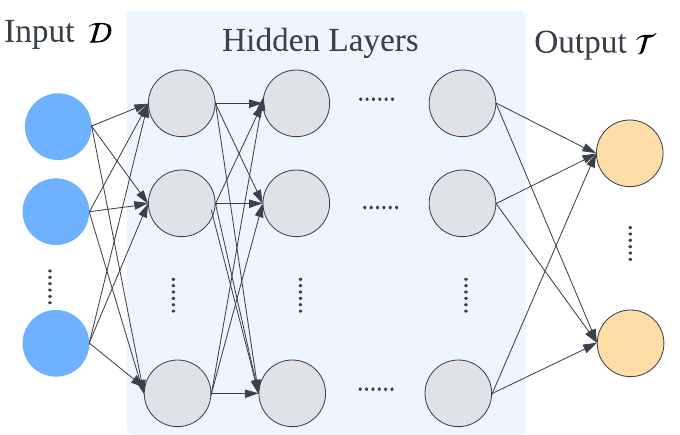}
    \caption{Fully connected network.}
    \label{fig:FNN}
\end{figure}

\section{Proposed algorithms} \label{sec: Methods}

\subsection{A learnable model for non-conservative forces}\label{subsec:LH1}

In standard LNNs and HNNs theory, non-conservative forces are assumed to be fully known and to be equal to actuation forces directly acting on the Lagrangian coordinates $q$. This is very restrictive, as already discussed in the introduction.

In this work, we include given by dissipation and actuator forces, i.e., ${F}_{\tt ext} = {F}_{\tt d}(q,\dot{q}) + {F}_{\tt a}({q})$. 
We propose the following model for dissipation forces
\begin{equation}
    {F}_{\tt d}(q,\dot{q}) = -{D}({q})\dot{{q}}, \label{fd}
\end{equation}
where ${D}({q}) \in \mathbb{R}^{N\times N}$ is the positive semi-definite damping matrix. Besides, we model the actuator force as
\begin{equation}
   {F}_{\tt a}({q}) = {A}({q}){u}, \label{fext}
\end{equation}
where ${u} \in \mathbb{R}^M$ is the control input signal to the system, and ${A}({q}) \in \mathbb{R}^ {N \times M}$ is an input transformation matrix. For example, ${A}$ could be the transpose Jacobian associated with the point of application of an actuation force on the structure.
With this model, we take into account that in complex robotic systems, actuators are, in general, not collocated on the measured configurations $q$. Note that, even if we accepted to impose an opportune change of coordinates, for some systems, a representation without ${A}$ is not even admissible \cite{pustina2023collocated}. With \eqref{fext}, we also seemingly treat underactuated systems.

Note that \cite{lutter2019bdeep} uses a dissipative model, but considers it in a white box fashion.

Hence, we rewrite the Lagrangian dynamics as follows
\begin{equation}
        \ddot{{q}} = \left (\frac{\partial^2 L({q}, \dot{{q}})}{\partial \dot{{q}}^2}\right)^{-1}
\left ({A}({{q}}){u} -  \frac{\partial^2 L({q}, \dot{{q}})}{\partial {q}\partial \dot{{q}}}\dot{{q}} + \frac{\partial L({q}, \dot{{q}})}{\partial {q}} - {D}({q})\dot{{q}}\right).
\label{armdynamics}
\end{equation}

Similarly, the Hamiltonian takes the form
\begin{equation} \label{eq:hamiltonian_dyn}
\begin{bmatrix}\dot{{q}}\\\dot{{p}}\end{bmatrix} = 
\begin{bmatrix}
 0 & I\\
-I & -{D}({q})
\end{bmatrix}\begin{bmatrix}\frac{\partial H({q}, \dot{{q}})}{\partial {q}}\\\frac{\partial H({q}, \dot{{q}})}{\partial {p}}\end{bmatrix}
+\begin{bmatrix}
 0\\
{A}({q})
\end{bmatrix}{u}.
\end{equation} 

\subsection{Non-conservative non-collocated Lagrangian and Hamiltonian NNs with modified loss}\label{subsec:pinns}

Figure \ref{fig:HandL} reports the proposed network framework, which builds upon Lagrangian and Hamiltonian NNs discussed in Sec. \ref{subsec:LNNHNN1}. Our work incorporates the damping matrix network, input matrix network, and a modified loss function into the original framework. The damping matrix network is used to account for the dissipation forces in the system via \eqref{fd}, while the input matrix network corresponds to $A(q)$ in \eqref{fext}. We predict the next state by integrating \eqref{armdynamics} or \eqref{eq:hamiltonian_dyn} with the aid of the Runge-Kutta4 integrator. Clearly, different integration strategies could be used in its place.

\begin{figure}[ht]
\centering
\includegraphics[width=1.0\linewidth]{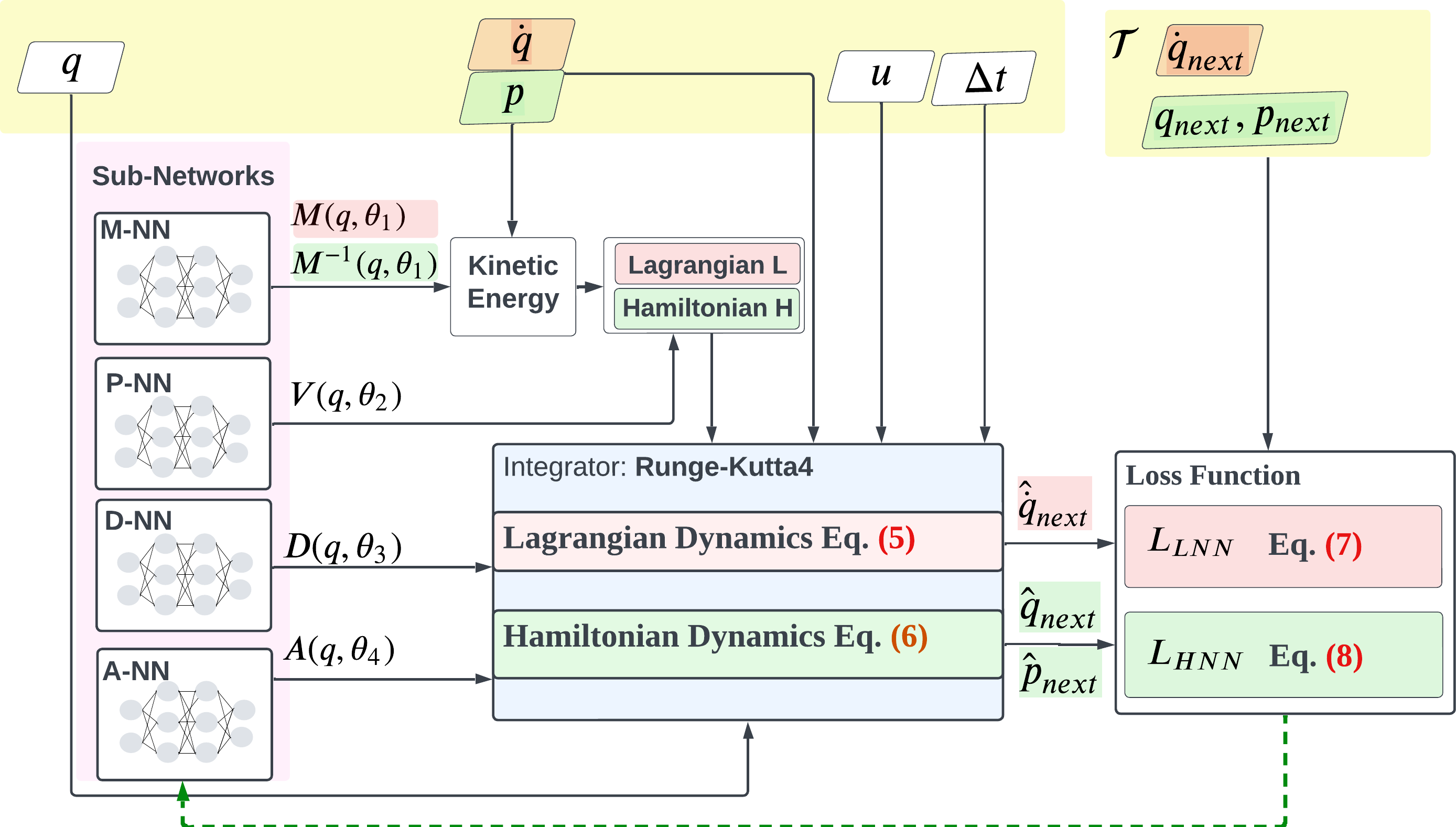}
\caption{The overview of Lagrangian and Hamiltonian neural networks: in red, the data and calculation process required for Lagrangian dynamics, while the green parts represent the corresponding data and calculation associated with the Hamiltonian dynamics.} 
\label{fig:HandL}
\end{figure}

The dataset $\mathfrak{D} = [\mathcal{D}_k, \mathcal{T}_k | k \in \{0, ..., k_{\tt end}\}]$  contains information about the state transitions of the mechanical system. With this compact notation, we refer not necessarily to a single evolution, but we include a concatenation of an arbitrary number of evolutions of the system. 
The input data $\mathcal{D}_k$ is composed of either $[{q}_k, {\dot{q}}_k, {u}_k, \Delta t]$, for Lagrangian dynamics, or $[{q}_k, {p}_k, {u}_k, \Delta t]$ in the case of Hamiltonian dynamics. Similarly, the corresponding label $\mathcal{T}_k$ is either ${\dot{q}}_{k+1}$, for the Lagrangian case, or $[{q}_{k+1}, {p}_{k+1}]$ for Hamiltonian dynamics.

 The values of ${M}({q})$, ${V}({q})$, ${D}({q})$, and ${A}({q})$ are estimated by four sub-networks, namely, the mass network (M-NN), potential energy network (P-NN), damping network (D-NN), and input matrix network (A-NN), as shown in Figure \ref{fig:HandL}. The kinetic energy can be calculated once the values of $\dot{{q}}$ or ${p}$ are obtained. Then, the Lagrangian or Hamiltonian functions can be derived from the kinetic and potential energies. The derivative of the states $\ddot{\hat{{q}}}$ or $[\dot{\hat{{q}}} \ \ \dot{\hat{{p}}}]^{\top}$ can be computed using (\ref{armdynamics}) or  (\ref{eq:hamiltonian_dyn}), respectively. The predicted next state $\dot{\hat{{q}}}$ or $[\hat{{q}} \ \ \hat{{p}}]^{\top}$ can be obtained using the Runge-Kutta4 integrator.

We thus employ the following modified losses \cite{gupta2019general,gupta2020structured}
\begin{equation}
    \mathcal{L}_{\mathrm{LNN}} = \frac{1}{\#\mathcal{D}}\sum_{k \in \mathcal{D}} \parallel \dot{{q}}_{k+1} -   \dot{\hat{{q}}} _{k+1}\parallel_2^2 
\end{equation}
for LNNs, where $\#\mathcal{D}$ is the cardinality of $\mathcal{D}$, and
\begin{equation}
    \mathcal{L}_{\mathrm{HNN}} = \frac{1}{\#\mathcal{D}}\sum_{k \in \mathcal{D}} \left( \parallel {q}_{k+1} -  \hat{{q}} _{k+1}\parallel_2^2 +
\parallel {p}_{k+1} -  \hat{{p}} _{k+1}\parallel_2^2 \right)
\end{equation}
for HNNs. Thus, compared to \eqref{eq:loss_LNN_vanilla} and \eqref{eq:loss_HNN_vanilla}, we are calculating the MSE of a future prediction of the state - simulated via the learned dynamics - rather than of the current accelerations, which cannot be measured.
Note that we also include a measure of the prediction error at the configuration level for $\mathcal{L}_{\mathrm{HNN}}$ because the information on $\frac{\partial H({q}, \dot{{q}})}{\partial {p}}$ appears disentangled from $D$ and $A$ (which are also learned) in the first $n$ equations of \eqref{eq:hamiltonian_dyn}.

\subsection{Sub-Network Structures}\label{subsec:pinns}
Constraints based on physical principles can be imposed on the parameters learned by the four sub-networks. Specifically, the mass and damping matrices must be positive definite and positive semi-definite, respectively. To this end, the network structure of the dissipation matrix can follow the prototype established for the mass matrix in \cite{lutter2019deep}. This structure can be decomposed into a lower triangular matrix ${L_D}$ with non-negative diagonal elements, which is then computed using the Cholesky decomposition \cite{trefethen2022numerical} as ${D} =  {L_D}{L_D}^{\top}$. The representation of ${D}({q})$ is illustrated in Figure \ref{fig:M-NND-NN}. 
\begin{figure}[h]
    \centering
    \includegraphics[width=0.4\columnwidth]{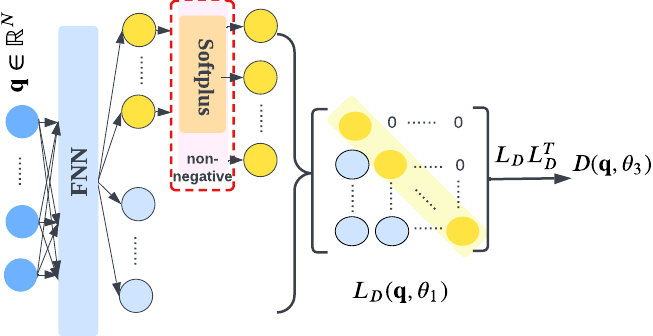}
    \caption{Diagram of the damping matrix including a feed-forward neural network, a non-negative shift for diagonal entries, and the Cholesky decomposition}
    \label{fig:M-NND-NN}
\end{figure}

The output of M-NN and D-NN is calculated as $(N^2+N)/2$, with the first $N$ values representing the diagonal entries of the lower triangular matrix. To ensure non-negativity, activation functions such as Softplus or ReLU are utilized as the last layer. Furthermore, a small positive shift, denoted by $ + \epsilon$, is introduced to guarantee that the mass matrix is positive definite. The remaining $(N^2-N)/2$ values are placed in the lower left corner of the lower triangular matrix.

The calculation of the potential energy is performed using a simple, fully connected network with a single output, which is represented as ${V}({q}, \theta_2)$. Moreover, A-NN, depicted in Figure \ref{fig:A-NN}, calculates ${A}({q}, \theta_4)$ with dimensions $\mathbb{R}^{N \times M}$.

\begin{figure}[h]
    \centering
    \includegraphics[width=0.4\columnwidth]{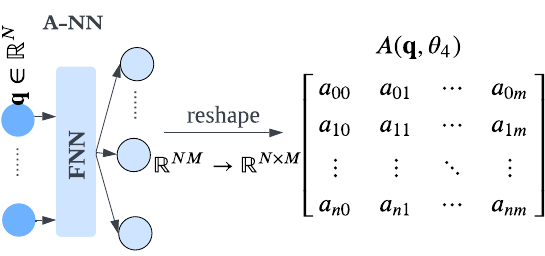}
    \caption{Diagram for actuator matrix: The fully connected network output is a vector in $\mathbb{R}^{NM}$, which is reshaped to a matrix in $\mathbb{R}^{N \times M}$. A sigmoid activation function can be applied to the matrix elements for value constraint.}
    \label{fig:A-NN}
\end{figure}

\subsection{PINN-based controllers}

We provide in this section two provably stable controllers by combining the learned dynamics in combination with classic model-based approaches. Before stating these results, it is important to spend a few lines remarking that a globally optimum execution of the learning process described above will result in learning $M_{\tt L}$, $G_{\tt L} = \frac{\partial V_{\tt L}}{\partial q}$, $A_{\tt L}$, $D_{\tt L}$ such that 
\begin{equation}
    \mathcal{L}(q,\dot{q}; M_{\tt L}, G_{\tt L}, A_{\tt L}, D_{\tt L}) = \mathcal{L}(q,\dot{q}; M, G, A, D)
\end{equation}
for the proposed LNN, or
\begin{equation}
    \mathcal{H}(q,\dot{q}; M_{\tt L}, G_{\tt L}, A_{\tt L}, D_{\tt L}) = \mathcal{H}(q,\dot{q}; M, G, A, D),
\end{equation}
for the proposed HNN, where $M, G, A, D$ are the \textit{real} reference values. Instead, we highlight the components that have been learned from the ones that are not by adding an L as a subscript. Also, by construction, $M_{\tt L}, G_{\tt L}, A_{\tt L}, D_{\tt L}$ will have all the usual properties that we expect from these terms, like $M_{\tt L}$ and $D_{\tt L}$ being symmetric and positive definite, and $G_{\tt L}$ being a potential force.

Yet, this does not imply that $M = M_{\tt L}$, $G = G_{\tt L}$, and so on. Indeed, there could exists a matrix $P(q)$ such that $P(q) M(q)$, $P(q) G(q)$, $P(q) A(q)$, $P(q) D(q)$ have all the properties discussed above while simultaneously fulfilling
\begin{equation}
    \mathcal{L}(q,\dot{q}; PM, PG, PA, PD) = \mathcal{L}(q,\dot{q}; M, G, A, D) \text{ or } \mathcal{H}(q,\dot{q}; PM, PG, PA, PD) = \mathcal{H}(q,\dot{q}; M, G, A, D).
\end{equation}

%Note that although it could be conceivable to have some different nonlinear transformations of the four terms $\mathcal{P}_{\tt M}(M)$, $\mathcal{P}_{\tt G}(G)$, $\mathcal{P}_{\tt A}(A)$, $\mathcal{P}_{\tt D}(D)$. Although we could not exclude that such an event can happen analytically, we never encountered it in our investigations. So, we will work under the assumption that a multiplicative the $P(q)$-case is the only one that can arise.
%
So controllers must be formulated and proofs derived under the assumption of the learned terms being close to the real ones up to a multiplicative factor. 

\subsubsection{Regulation}

The goal of the following controller is to stabilize a given configuration ${q}_{\tt ref}$
\begin{equation}\label{eq:controller_PD}
    {u} = {A}_{\tt L}^{+}({q}){G}_{\tt L}({q}) + {A}_{\tt L}^{\top}(q)({K}_{\tt P} ({q}_{\tt ref} - {q}) - {K}_{\tt D}\dot{{q}}),
\end{equation}
%where ${A}_{\tt L}({q}_{\tt ref})$ is the full rank input matrix directly learned by the LNN, and all elements are restricted to the value between 0 and 1; 
where we omit the arguments $t$ and $\theta_i$ to ease the readability. We highlight the components that have been learned from the ones that are not by adding an L as a subscript.
${G}_{\tt L}({q}_{\tt ref})$ is the potential force which can be calculated by taking the partial derivative of the potential energy learned by the LNN; ${K}_{\tt P}$ and ${K}_{\tt D}$ are positive definite control gains.

For the sake of conciseness, we introduce the controller, and we prove its stability for the fully actuated case. However, the controller and the proof can be extended to the generic underactuated case using arguments in \cite{della2023model}. This will be the focus of future work.
\begin{proposition}\label{pr:1}
    Assume that $M = N$, with $A$ and $A_{\tt L}$ both full rank. Then, given a maximum admitted error $\delta_{\tt q}$, the closed loop of \eqref{armdynamics} and \eqref{eq:controller_PD} is such that 
    \begin{equation}
        \lim_{{t \to \infty}} q(t) = q_{\tt ss} \text{ with } ||q_{\tt ss} - q_{\tt ref }|| < \delta_{\tt q},
    \end{equation}
    if ${K}_{\tt P},{K}_{\tt D} \succ \kappa I$, with $\kappa \in \mathbb{R}$ high enough, and if it exists a matrix ${P}(q) \in \mathbb{R}^{N \times N}$ such that $|| {G}_{\tt L} (q) - {P}(q){G}(q) || < \delta_{\tt G}$, for some finite and positive $\delta_{\tt G}$.
    Also, we assume that
    \begin{equation}\label{eq:PD_transpose}
        A(q)[{A}_{\tt L} (q) - {P}(q) {A}(q)]^{\top} + A(q)A^{\top}(q)P^{\top}(q) \succ 0,
    \end{equation}
    and that
    \begin{equation}\label{eq:smallenss_A_2}
        ||{A}^{-1}(q) {P}^{-1}(q) [{A}_{\tt L} (q) - {P}(q) {A}(q)]|| < 1.
    \end{equation}
\end{proposition}
\begin{remark}\label{rk:1}
    Note that if $P(q) \succ 0$, then $A(q)A^{\top}(q)P^{\top}(q) \succ 0$, and \eqref{eq:PD_transpose} translates into another request of ${A}_{\tt L} (q)$ being close enough to ${A}(q)$ up to a multiplicative factor ${P}(q)$. The positive definiteness of $P$ is, in turn, a request on the quality of the outcome of the learning process. Indeed, if $|| {M}_{\tt L} (q) - {P}(q) {M}(q) ||$ is small enough, then the positive definitness of ${M}_{\tt L} $ and $M$ implies the one of $P$.
    Similarly, \eqref{eq:smallenss_A_2} is always verified for small enough $||{A}_{\tt L} (q) - {P}(q) {A}(q)||$.
\end{remark}
\begin{proof}
    Let us introduce the matrix $\Delta_{\tt A} \in \mathbb{R}^{N \times N}$ such that $ {A}_{\tt L} (q) = {P}(q) {A}(q) + \Delta_{\tt A}(q)$. This matrix is small enough by hypothesis as detailed in Remark \ref{rk:1}. We now want to bound the difference between the inverse of ${A}(q)$ and ${P}(q) {A}_{\tt L} (q)$. The goal is to write ${A}_{\tt L}^{-1}(q) =   ({P}(q){A}(q))^{-1} + \Delta_{\tt I}(q)$, with $||\Delta_{\tt I}(q)|| < \delta_{\tt I}$.

    %%%%%
    %------
    %
    %Given ${A}_{\tt L} (q) = {P}(q) {A}(q) + \Delta{\tt A}(q)$ with $||\Delta_{\tt A}(q)|| < \delta_{\tt A}$, we wish to express ${A}{\tt L}^{-1}(q)$ in terms of $({P}(q){A}(q))^{-1}$ and a perturbation $\Delta{\tt I}(q)$.

    Under hypothesis \eqref{eq:smallenss_A_2}, we can use the Neumann series \cite{petersen2008matrix}
    \begin{equation}
    {A}_{\tt L}^{-1}(q) = ({P}(q){A}(q) + \Delta_{\tt A}(q))^{-1} = ({P}(q){A}(q))^{-1} - ({P}(q){A}(q))^{-1} \Delta_{\tt A}(q) ({P}(q){A}(q))^{-1} + \dots
    \end{equation}
    %s
    Rearranging terms, we define
    \begin{equation}
    \Delta_{\tt I}(q) = {A}_{\tt L}^{-1}(q) - ({P}(q){A}(q))^{-1} = - ({P}(q){A}(q))^{-1}  \Delta_{\tt A}(q) ({P}(q){A}(q))^{-1} + \ldots
    \end{equation}
    We can therefore bound the norm of $\Delta_{\tt I}(q)$ as follows
    \begin{equation}
    ||\Delta_{\tt I}(q)|| \leq \frac{||({P}(q){A}(q))^{-1} \Delta_{\tt A}(q) ({P}(q){A}(q))^{-1}||}{1 - ||({P}(q){A}(q))^{-1} \Delta_{\tt A}(q)||} < \delta_{\tt I}.
    \end{equation}
    We can therefore, rewrite the generalized forces produced by the controller $A(q) u$ as
    \begin{equation}%\small
        \begin{split}
            &A(q) \left[{A}_{\tt L}^{-1}(q){G}_{\tt L}({q}) + {A}_{\tt L}^{\top}(q)({K}_{\tt P} ({q}_{\tt ref} - {q}) - {K}_{\tt D}\dot{{q}}) \right]\\
            = \; &A(q)(A^{-1}(q)P^{-1}(q) + \Delta_{\tt I}(q))\left[(P(q){G}({q}) + \Delta_{\tt G}(q))\right] + A(q)(P(q)A(q) + \Delta_{\tt A}(q))^{\top}\left[{K}_{\tt P} ({q}_{\tt ref} - {q}) - {K}_{\tt D}\dot{{q}} \right]\\
            = \; &(P^{-1}(q) + A(q)\Delta_{\tt I}(q))(P(q){G}({q}) + \Delta_{\tt G}(q)) + (A(q)A^{\top}(q)P^{\top}(q) + A(q)\Delta_{\tt A}^{\top}(q))\left[{K}_{\tt P} ({q}_{\tt ref} - {q}) - {K}_{\tt D}\dot{{q}} \right]\\
            = \; &G(q) + \Delta_{\tt all}(q) + \hat{K}_{\tt P} ({q}_{\tt ref} - {q}) - \hat{K}_{\tt D}\dot{{q}}
        \end{split}
    \end{equation}

    Where $\Delta_{\tt all}(q) = P^{-1}(q)\Delta_{\tt G}(q) + A(q)\Delta_{\tt I}(q)P(q)G(q) + A(q)\Delta_{\tt I}(q)\Delta_{\tt G}(q)$ is a bounded term, as sum and product of bounded terms. The gains $\hat{K}_{\tt P}$ and $\hat{K}_{\tt D}$ are positive definite being product of two positive definite matrices.
    The closed loop is then
    \begin{equation}\label{eq:closed_loop_PD}
        \begin{split}
            M(q)\Ddot{q} + C(q,\dot{q})\dot{q} = \Delta_{\tt all}(q) + \hat{K}_{\tt P} ({q}_{\tt ref} - {q}) - (D(q) + \hat{K}_{\tt D})\dot{{q}}.
        \end{split}
    \end{equation}
    In this segment, we establish our thesis by adopting and replicating the arguments provided in \cite{montagna2023regulation}, which is, in turn, adapted from the seminal paper \cite{tomei1991simple}. This direct application of an existing theorem has been made possible by our rearrangement of the closed loop, which has made it identical to the structure delineated in those papers.
    
\end{proof}
Note that even if we provided the proof using a Lagrangian formalism, the Hamiltonian version can be derived following similar steps. Also, note that the bounds on the learned matrices are always verified for any choice of $\delta_{\tt A},\delta_{\tt G}$ at the cost of training the model with a large enough training set.
We conclude with a corollary that discusses the perfect learning scenario.
\begin{corollary}\label{cr:1}
    Assume that $M = N$ and $A$ is full rank. Then, the closed loop of \eqref{armdynamics} and \eqref{eq:controller_PD} is such that 
    \begin{equation}
        \lim_{{t \to \infty}} q(t) = q_{\tt ref },
    \end{equation}
    if ${K}_{\tt P},{K}_{\tt D} \succ 0$ and if it exists a matrix ${P}(q) \in \mathbb{R}^{N \times N}$ such that ${M}_{\tt L} (q) = {P}(q) {M}(q)$, ${A}_{\tt L} (q) = {P}(q) {A}(q)$, ${G}_{\tt L} (q) = {P}(q){G}(q) $.
\end{corollary}
\begin{proof}
    Let's start from \eqref{eq:PD_transpose}, which now becomes
    %
    %\begin{equation}
        $A(q)A^{\top}(q)P^{\top}(q) \succ 0$.
    %\end{equation}
    %
    Furthermore, considering that $A(q)$ is full rank by hypothesis yields the equivalent condition $P^{\top}(q) \succ 0$. As discussed in the remark before, this is implied by the fact that ${M}_{\tt L} (q) = {P}(q) {M}(q)$ and both ${M}_{\tt L}$ and ${M}(q)$ are positive definite. Thus, \eqref{eq:PD_transpose} is always verified.
    Similarly, \eqref{eq:smallenss_A_2} is trivially verified for ${A}_{\tt L} (q) = {P}(q) {A}(q)$.
    
    Moreover, note that $\Delta_{\tt all} = 0$ as the deltas are now all zero. So, the closed loop \eqref{eq:closed_loop_PD} is always the equivalent of a mechanical system, without any potential force, controlled by a PD. Note that the gains are positive because we just proved that $P^{\top}(q) \succ 0$, and because ${K}_{\tt P},{K}_{\tt D} \succ 0$ by hypothesis. The proof of stability follows standard Lyapunov arguments (see, for example, \cite{murray1994mathematical}) by using the Lyapunov candidate $V(q,\dot{q}) = T(q,\dot{q}) + \frac{1}{2}{q}^{\top}\hat{K}_{\tt P}{q}$.
\end{proof}

\subsubsection{Trajectory tracking}
The goal of the following controller is to track a given trajectory in configuration space ${{q}}_{\tt ref}:\mathbb{R} \rightarrow \mathbb{R}^n$. We assume ${{q}}_{\tt ref}$ to be bounded with bounded derivatives. We also assume the system to be fully actuated - i.e., $M=N$, $\mathrm{det}({A}) \neq 0$, $\mathrm{det}({A}_{\tt L}) \neq 0$. Under these assumptions, we extend \eqref{eq:controller_PD} with the following controller to follow the desired trajectory 
\begin{equation}
    \begin{split}
        {u} = &{A}_{\tt L}^{-1}(q) \left({M}_{\tt L}({q}_{\tt ref})\ddot{{q}}_{\tt ref} + {C}_{\tt L} ({q}_{\tt ref}, \dot{{q}}_{\tt ref})\dot{{q}}_{\tt ref} + {D}_{\tt L} (q_{\tt ref})\dot{{q}}_{\tt ref} + {G}_{\tt L}({q}_{\tt ref})\right) \\
    + &{A}_{\tt L}^{\top}(q) \left({K}_{\tt P} ({q}_{\tt ref} - {q}) + {K}_{\tt D}(\dot{{q}}_{\tt ref}-\dot{{q}}))\right),
    \end{split}
\label{eq: controller}
\end{equation}
where we omit the arguments $t$ and $\theta_i$ to ease the readability. We highlight the components that have been learned from the ones that are not by adding an L as a subscript.
We can obtain the Coriolis matrix ${C}_{\tt L} ({q}_{\tt ref}, \dot{{q}}_{\tt ref})$ from the learned Lagrangian by taking the second partial derivative of the Lagrangian with respect to the desired joint position ${q_{\tt ref}}$ and velocity $\dot{{q}}_{\tt ref}$, i.e., $\frac{\partial^2 L({q_{\tt ref}}, \dot{{q}}_{\tt ref})}{\partial {q_{\tt ref}}\partial \dot{{q}}_{\tt ref}}$.

\begin{corollary}
    The closed loop of \eqref{armdynamics} and \eqref{eq: controller} is such that, for some $ \delta_{\tt q} \geq 0$,
    \begin{equation}
        \lim_{{t \to \infty}} ||q(t)  - q_{\tt ref }(t)|| < \delta_{\tt q},
    \end{equation}
    if ${K}_{\tt P},{K}_{\tt D} \succ \kappa I$, with $\kappa \in \mathbb{R}$ high enough, and if it exists a matrix ${P}(q) \in \mathbb{R}^{N \times N}$ such that ${A}_{\tt L} (q) = {P}(q) {A}(q)$, $ {M}_{\tt L} (q) = {P}(q) {M}(q)$, $ {C}_{\tt L} (q) = {P}(q) {C}(q)$, $ {G}_{\tt L} (q) = {P}(q) {G}(q)$, ${D}_{\tt L} (q) = {P}(q) {D}(q)$. We also assume that $P$ is such that $||P^{-1}(q)P(q_{\tt ref}) - I|| < \delta_{\tt P}$ for some $\delta_{\tt P}>0$.
\end{corollary}
\begin{proof}
    %Proof of the proposition is coming soon.
    We can rewrite \eqref{eq: controller} by substituting the values of the learned elements in terms of $P$. The result is
    \begin{equation}
    \begin{split}
        A(q) {u} = & \Delta_{\tt all} + \left({M}({q}_{\tt ref})\ddot{{q}}_{\tt ref} + {C} ({q}_{\tt ref}, \dot{{q}}_{\tt ref})\dot{{q}}_{\tt ref} + {D} (q_{\tt ref})\dot{{q}}_{\tt ref} + {G}({q}_{\tt ref})\right) \\
    + &(A(q) {A}^{\top}(q)P^{\top}(q)) \left({K}_{\tt P} ({q}_{\tt ref} - {q}) + {K}_{\tt D}(\dot{{q}}_{\tt ref}-\dot{{q}}))\right),
    \end{split}
\end{equation}
where 
$$\Delta_{\tt all} = \Delta_{\tt P}\left({M}({q}_{\tt ref})\ddot{{q}}_{\tt ref} + {C} ({q}_{\tt ref}, \dot{{q}}_{\tt ref})\dot{{q}}_{\tt ref} + {D} (q_{\tt ref})\dot{{q}}_{\tt ref} + {G}({q}_{\tt ref})\right), $$ 
with $\Delta_{\tt P} = P^{-1}(q)P(q_{\tt ref}) - I$. Thus, $\Delta_{\tt all}$ is bounded by hypothesis as a product and the sum of bounded terms.
Moreover, as discussed in the proof of Corollary \ref{cr:1}, $A {A}^{\top}P^{\top} \succ 0$. Thus, being the closed loop is equivalent to the one discussed in \cite{kelly1994pd}, the same steps discussed there can be followed to yield the proof.
\end{proof}
Note that even if we provided the proof using a Lagrangian formalism, the Hamiltonian version can be derived following similar steps. Also, the bound $\delta_{\tt q}$ can be made as small as we desire at the cost of making the control gains large enough.

Finally, note that we provided here only proof of stability for the perfectly learned case. Similar hypotheses and arguments to the ones in Proposition \ref{pr:1} would lead to similar results in the tracking case, with $||{P}(q) {A}_{\tt L} (q) - {A}(q)|| < \delta_{\tt A}$, $|| {P}(q) {M}_{\tt L} (q) - {M}(q) || < \delta_{\tt M}$, $|| {P}(q) {C}_{\tt L} (q) - {C}(q) || < \delta_{\tt C}$, $|| {P}(q) {G}_{\tt L} (q) - {G}(q) || < \delta_{\tt G}$, $|| {P}(q) {D}_{\tt L} (q) - {D}(q) || < \delta_{\tt G}$, for some finite and positive $\delta_{\tt A},\delta_{\tt M},\delta_{\tt C},\delta_{\tt G},\delta_{\tt D} \in \mathbb{R}$. 

\section{Methods: Simulation and experiment design}\label{subsec:Experiment setup}
To evaluate the efficacy of the proposed PINNs and PINN-based control, we apply them in three distinct tasks: {(T1)} Learning the dynamic model of a one-segment spatial soft manipulator, (T2) Learning the dynamic model of a two-segment spatial soft manipulator, (T3) Learning the dynamic model of the Franka Emika Panda robot. % in a simulation environment using PyBullet. 
We selected (T1) and (T2) because they have a nontrivial $A(q)$, and (T3) because it has several degrees of freedom. %To this end, we consider velocity data for the LNN and momentum data for the HNN.
Furthermore, we employ the learned dynamics to design and test model-based controllers for T2 and T3. 

In a hardware experiment, the LNN is utilized to learn the dynamic model of the tendon-driven soft manipulator reported in \cite{9762144} and the Panda robot. We show for the first time experimental closed-loop control of a robotic system (the Panda robot) with a PINN-based algorithm.

\subsection{Data Generation} 
Training data for T1 and T2 are generated by simulating the dynamics of one-segment and two-segment soft manipulators in MATLAB. For T1, ten different initial states are combined with ten different input signals to generate data using the one-segment manipulator dynamics model. Each combination produces ten-second training data with a time step of 0.0002 seconds. For T2, we use a variable step size in Simulink to generate datasets from the mathematical model of a two-segment soft manipulator. With this approach, we create twelve different sixty-second trajectories, which are subsequently resampled at fixed frequencies of 50Hz, 100Hz, and 1000Hz. Concerning T3, PyBullet simulation environment  is used to generate training data corresponding to the Panda robot. Then, Different input signals are applied to the joints to create the data of 70 different trajectories with a frequency of 1000Hz.

Regarding experimental validation, we propose the following experiments. For the tendon-driven continuum robot, we provide sinusoidal inputs with different frequencies and amplitudes to the actuators---four motors---and record the movement of the robot. An IMU records the tip orientation data with a 10Hz sampling frequency. As a result, 122 trajectories are generated, and four more are collected as the test set. For the Panda robot, we provide 70 sets of sinusoidal desired joint angles with different amplitudes and frequencies. We collect the torque, joint angle, and angular velocity data using the integrated sensors, considering a sampling frequency of 500Hz.

\subsection{Baseline Model and Model Training}
In order to provide a basis for comparison, baseline models are established for all simulations and hardware experiments. These models, which serve as a control, are constructed using fully connected network and trained using the same datasets as the proposed models, however, with a larger amount of data and a greater number of training epochs. These baseline models aimed to demonstrate the benefits of incorporating physical knowledge into neural networks.

In this project, all the neural networks utilized are constructed using the JAX and dm-Haiku packages in Python. In particular, the JAX Autodiff system is used to calculate partial derivatives and the Hessian within the loss function. The optimization of the model parameters is carried out using AdamW in the Optax package.

\section{Simulation Results} \label{sec: Simulation Results}
\subsection{One-segment 3D soft manipulator}
To define the configuration space of the soft manipulator, we adopt the piecewise constant curvature (PCC) approximation \cite{hannan2003kinematics} shown in Figure \ref{fig: PCC}. Customarily, this approximation describes the configuration of each segment as ${q}_i = [\phi_i, \theta_i, \delta \ell_i]$, where $\phi_i$ is the plane orientation, $\theta_i$ is the curvature in that plane, and $\delta \ell_i$ is the change of arc length. In this work, the configuration-defined method reported in \cite{della2020improved} is used to avoid the singularity problem of PCC. Hence, the configuration of each segment is given by  $[\Delta_{xi}, \Delta_{yi}, \Delta_{\ell i}]$, where $\Delta_{xi}$ and $\Delta_{yi}$ are the difference of arc length.
\begin{figure}[t]
    \centering
    \includegraphics[width=0.6\columnwidth]{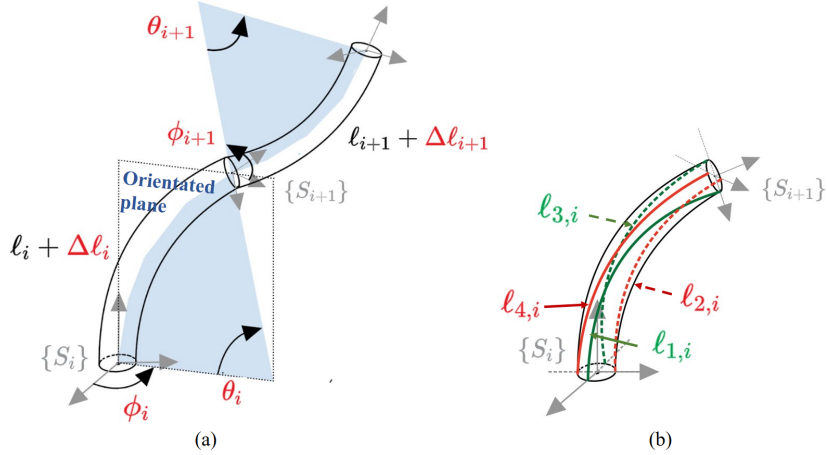}
    \caption{PCC approach illustration: (a) two-segment soft manipulator is shown, where $S_{i}$ is the end frame, the blue parts are the orientated plane, $\ell_i$ is the original length of each segment; (b) shows the length
    of the four arcs whose ends connected to the frame $S_{i}$ }
    \label{fig: PCC}
\end{figure}

\begin{table}[ht]
\caption{One-segment soft manipulator simulation detailed information}
\centering
\begin{tabular}{cccc} 
\hline
                 & Black-box model           &  \makecell[c]{Lagrangian-based \\  learned model}                       & \makecell[c]{Hamiltonian-based \\  learned model}                      \\ 
\hline
model (width $\times$ depth)            & $128 \! \times \! 5$            & $ 32 \!\times \! 3, 5 \! \times \! 3,  16 \! \times \! 2$ & $ 32 \!\times \! 3, 5 \! \times \! 3,  16 \! \times \! 2$   \\
sample number    & 19188 & 8000                            & 8000                           \\
training epoch   & 15000 & 6000                            & 6000                            \\
traning error    & $6.891e^{-5} \pm  4.63e^{-4}$        & $8.418e^{-7} \pm  1.77e^{-5}$                                & $5.374e^{-11} \pm  7.74e^{-10}$                                 \\
prediction error [m] & $7.647 \pm  10.413$(5s)        & $0.171 \pm  0.272$(5s)                                   & $0.0220 \pm  0.0210$ (5s)                                 \\
\hline
\end{tabular}
\label{table: one-segment-simulation}
\end{table}

\begin{figure}[ht]
    \centering
    \includegraphics[width=0.7\columnwidth]{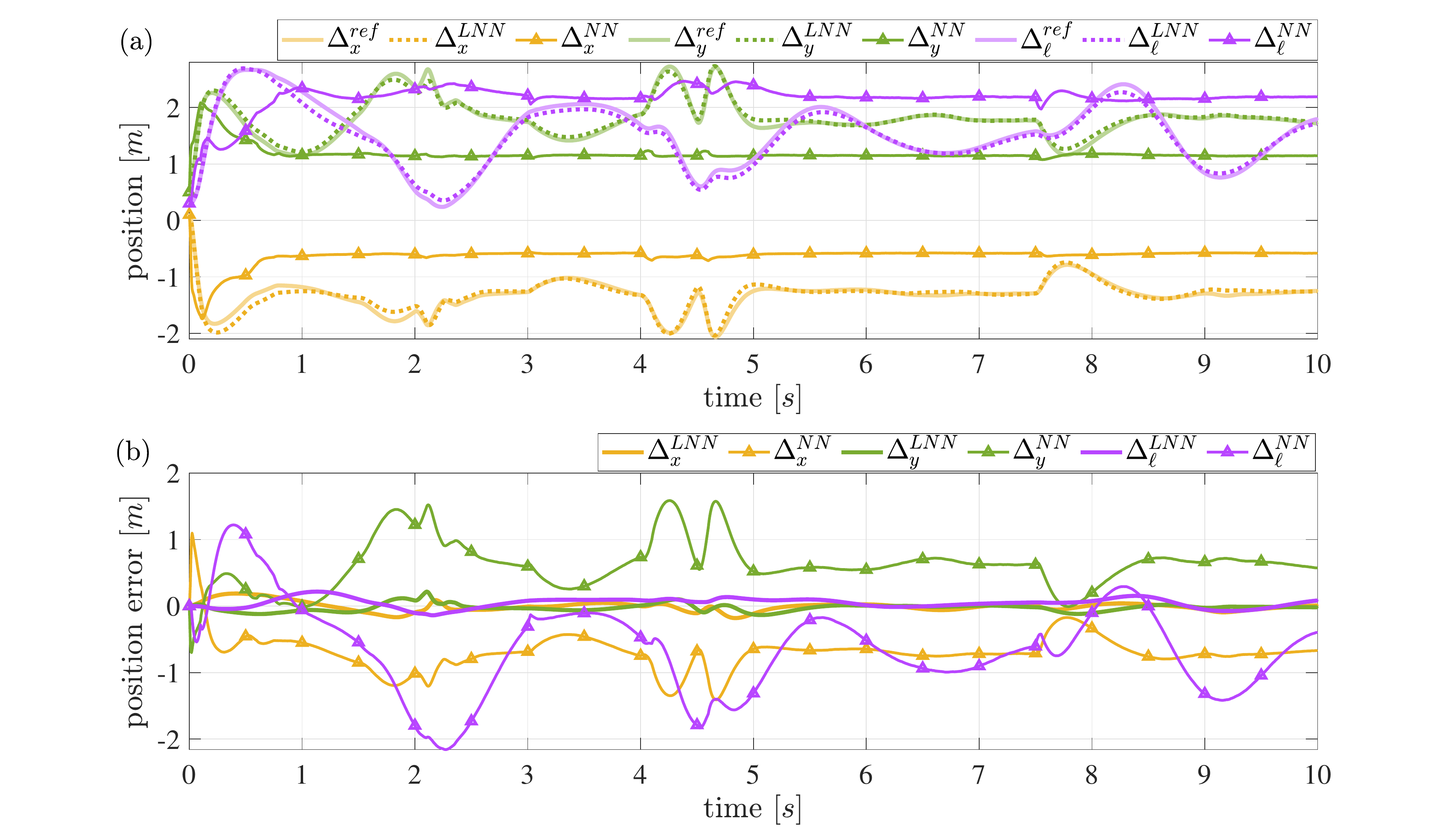}
    \caption{One-segment soft manipulator leaned model comparison results: (a) depicts the predictions generated by the black-box model ($\bigtriangleup$), the Lagrangian-based learning model ($\cdots $), and the ground-truth ($-$) arising from the dynamic mathematical equations; (b) shows the prediction error of these two learned models.}
    \label{fig: one_segment_LNN}
\end{figure}
The detailed information for this task is shown in Table \ref{table: one-segment-simulation}. The prediction results of these two learned models are compared in Figure \ref{fig: one_segment_LNN}. The figure indicates that the model trained by LNNs exhibits a high degree of predictive accuracy, manifesting near-infinite prediction capabilities with over 50,000 consecutive prediction steps in this example. While some areas exhibit less precise fits, it is important to note that such errors do not accrue over time. These outcomes suggest that LNN-based models can effectively capture the underlying dynamics of the one-segment soft manipulator. By contrast, the black-box model converges during the training process, but it does not gain insights into the dynamic model from its prediction performance. This system is also learned using HNNs by providing momentum data. Hamiltonian-based neural networks yield similar quality prediction results as Lagrangian-based neural networks, as shown in Figure \ref{fig: one_segment_LNN_HNN}. 

\begin{figure}[!htb]
    \centering
    \includegraphics[width=0.7\columnwidth]{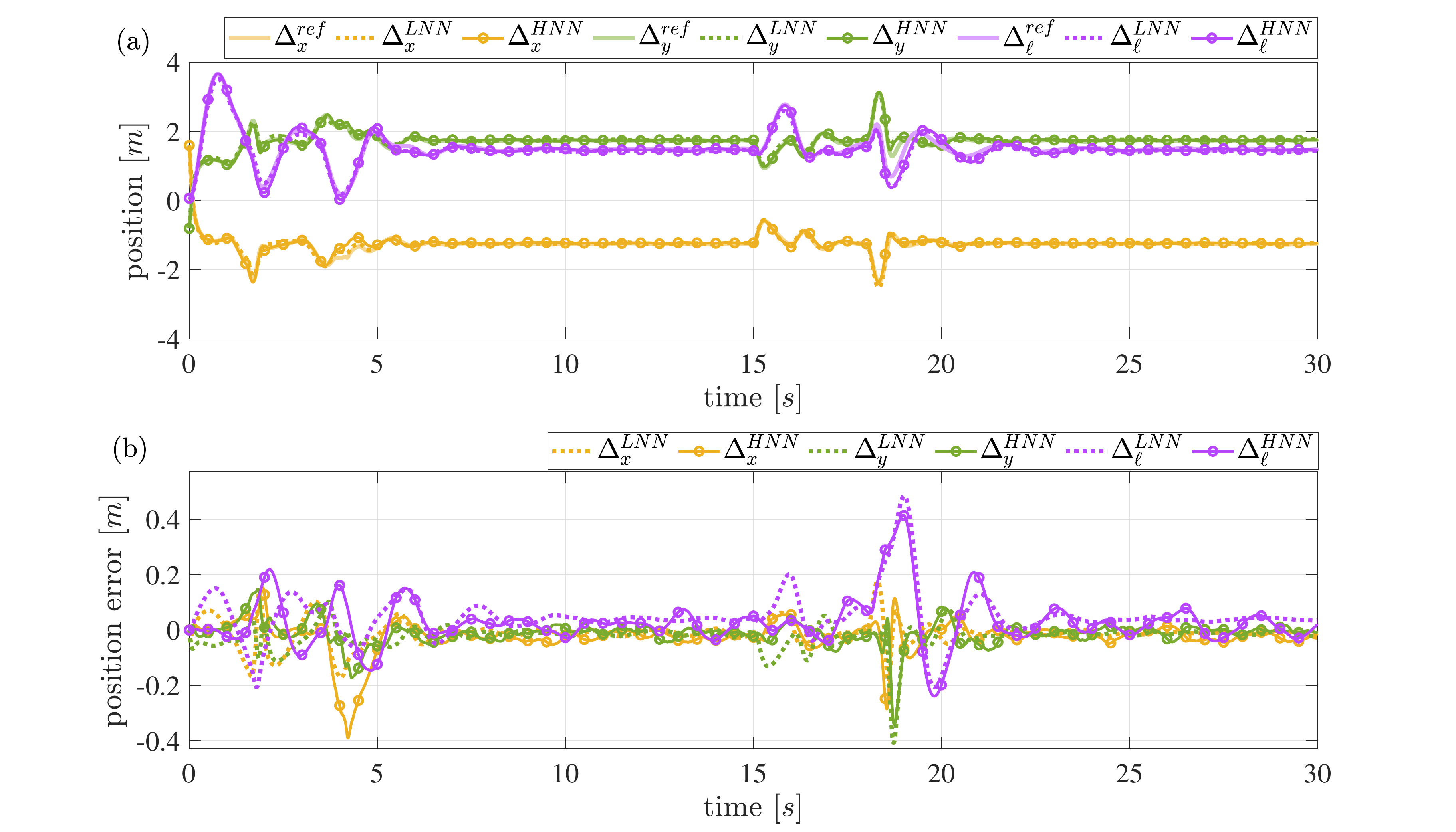}
    \caption{One-segment soft manipulator HNN and LNN comparison: (a) shows the Lagrangian-based learned model prediction results ($\cdots$), Hamiltonian-based learned model prediction results ($\circ$), and the ground-truth prediction ($-$); (b) error of the two models with the ground truth.}
    \label{fig: one_segment_LNN_HNN}
\end{figure}

\begin{table}[!htb]
\centering
\caption{Mathematical model matrices of one-segment soft manipulator}
\label{Task 3: Mathematical model matrices}
\scalebox{0.8}{
\begin{tabular}{cccccc}
\hline
$q$                                                   & $M(q)$                                                                                                                                            & $M^{-1}(q)$                                                                                                       & $D(q) $                                                                                    & $G(q) $                                                 & $A(q) $                                                                                                                                                                                                                  \\ 
\hline
$\begin{bmatrix} 1.20\\-0.20\\0.15 \end{bmatrix}$  & $\begin{bmatrix} 1.73e^{-3} & -3.12e^{-5} & -1.96e^{-3} \\ -3.12e^{-5} & 1.55e^{-3} & 3.26e^{-4} \\ -1.96e^{-3} &  3.26e^{-4} & 9.29e^{-2} \end{bmatrix}$ & $\begin{bmatrix} 593.09 &    9.35 &12.47\\ 9.35 & 647.61 & -2.08 \\ 12.47 & -2.08 &  11.04 \end{bmatrix}$& \multirow{4}{*}{$\begin{bmatrix}0.1 &  0 & 0 \\0 & 0.1 & 0 \\0 & 0 & 0.1 \end{bmatrix}$} & $\begin{bmatrix} 1.29\\-0.22\\-1.15 \end{bmatrix}$ & $\begin{bmatrix} -0.04 & -1.0 &  0.07 \\ 0.78 &  0.04 & -0.01 \\ 0. &     0. &   0.77 \end{bmatrix}$        \\ 
$\begin{bmatrix} \; \;0.80 \;\;\\ \; \;0.20 \;\;\\ \; \;0.30 \;\;\end{bmatrix}$    & $\begin{bmatrix}3.64e^{-3} &  4.52e^{-5} & -1.94e^{-3}\\ 4.52e^{-5} & 3.47e^{-3} & -4.84e^{-4} \\ -1.94e^{-3} & -4.84e^{-4} &  9.67e^{-2} \end{bmatrix}$  & $\begin{bmatrix}277.76 &  -2.84 & 5.55\\ -2.84 & 288.42 &  1.39\\ 5.55 & 1.39 &  10.46 \end{bmatrix}$&                                                                                          & $\begin{bmatrix}0.89\\ 0.22 \\-1.09 \end{bmatrix}$ & $\begin{bmatrix} 0.03 & -0.99 & 0.06 \\ 0.90 & -0.03 &  0.02 \\ 0.   &       0.   &       0.89 \end{bmatrix}$ \\ 
\hline
\end{tabular}}
\end{table}

\begin{table}[!htb]
\centering
\caption{Lagrangian-based learning model matrices of one-segment soft manipulator}
\label{Task 3: Physics-based learning model matrices (Lagrangian learning model)}
\scalebox{0.8}{
\begin{tabular}{cccccc} 
\hline
q & $\hat{M}(q)$ & $\hat{D}(q)$  & $\hat{G}(q)$ & $\hat{A}(q)$  &  $P(q)$ \\\hline
$\begin{bmatrix}
1.20 \\ -0.20 \\ 0.15 \end{bmatrix}$  & $\begin{bmatrix} 
4.23e^{-3} & 1.20e^{-3} & -0.03 \\ 1.20e^{-3} & 5.99e^{-3} & -0.02\\ -0.03 & -0.02 & 0.59 \end{bmatrix}$ & $\begin{bmatrix} 
0.16 & -0.02 & 0.0 \\ -0.02 & 0.33 & -0.01 \\ 0.0 & -0.01 & 0.35 \end{bmatrix}$ & $\begin{bmatrix} 2.44\\-0.61\\ -5.25 \end{bmatrix}$  & $\begin{bmatrix}
0.12 & -1.72& -0.21\\ 3.05 & -0.19 & -0.13 \\ -0.34 & 1.01 & 3.40 
\end{bmatrix}$ & $\begin{bmatrix}
0.61 & -0.02 & 0.03 \\-0.02 & 0.28 & 0.01\\ 0.33 & 0.15 & 0.25
\end{bmatrix}$    \\
$\begin{bmatrix} \; \;0.80 \;\;\\ \; \;0.20 \;\;\\ \; \;0.30 \;\;\end{bmatrix}$    & $\begin{bmatrix} 6.93e^{-3} & 1.84e^{-3} & -0.03 \\ 1.84e^{-3} & 0.01 & -0.02 \\-0.03 & -0.02 & 0.50 \end{bmatrix}$ & $\begin{bmatrix} 0.17 & -0.01 & -0.0 \\ -0.01 & 0.33 & -0.01 \\ -0.0 & -0.01 & 0.35 \end{bmatrix}$ & $\begin{bmatrix} 1.62 \\ 0.81\\ -4.67 \end{bmatrix}$  & $\begin{bmatrix} 0.19 &  -1.66 & -0.20\\ 2.97 & -0.25 & -0.13 \\ -0.40 & 1.01 & 3.43 \end{bmatrix}$  & $\begin{bmatrix} 
0.62 & -0.02 & 0.03 \\ -0.02 & 0.31 & 0.01 \\ 0.21 & 0.10 & 0.26 \end{bmatrix}$  \\\hline
\end{tabular}}
\end{table}

\begin{table}[ht]
\centering
\caption{Hamiltonian-based learning model matrices of one-segment soft manipulator}
\label{Task 3:Physics-based learning model matrices (Hamiltonian learning model)}
\scalebox{0.8}{
\begin{tabular}{ccccc} 
\hline
q                                                  &$\hat{M}^{-1}(q)$                                                                                                                                        & $\hat{D}(q)$                                                                                                           & $\hat{G}(q)$                                                      & $\hat{A}(q)$                                                                                                                                                                                            \\ 
\hline
$\begin{bmatrix} 1.20 \\ -0.20 \\ 0.15 \end{bmatrix}$  &$\begin{bmatrix} 600.32 & 16.90&   15.67 \\ 16.90 &  622.92 & -1.34\\ 15.67 & -1.34 &  11.61 \end{bmatrix}$   & $\begin{bmatrix} 1.02e^{-1} & 3.44e^{-3} &  8.12e^{-5} \\3.44e^{-3} & 1.05e^{-1} & -4.39e^{-4} \\ 8.12e^{-5} & -4.39e^{-4} &  9.91e^{-2} \end{bmatrix}$ & $\begin{bmatrix} 1.33\\-0.18\\-1.15 \end{bmatrix}$ & $\begin{bmatrix} -0.06 & -0.94 & 0.05 \\ 0.83 &   0.02 & -0.04 \\ 0.0 &  0.01&  0.78 \end{bmatrix}$                                  \\
$\begin{bmatrix} \; \;0.80 \;\;\\ \; \;0.20 \;\;\\ \; \;0.30 \;\;\end{bmatrix}$    &$\begin{bmatrix} 285.01 &    11.08 &   6.65\\11.08 &  292.46 & 2.06\\ 6.65&   2.06 & 10.59\end{bmatrix}$      & $\begin{bmatrix} 1.01e^{-1} & 3.48e^{-3} & 6.56e^{-4}\\3.48e^{-3} & 1.03e^{-1} & -7.45e^{-5} \\ 6.56e^{-4} & -7.45e^{-5} &  9.87e^{-2}\end{bmatrix}$      & $\begin{bmatrix} 0.93 \\ 0.25 \\ -1.10 \end{bmatrix}$ & $\begin{bmatrix} 0.03 & -0.96 &  0.05\\ 0.92 & -0.03 & -0.02\\-0.01 &  0.0 & 0.89 \end{bmatrix}$  \\
\hline
\end{tabular}}
\end{table}

The matrices obtained from these two physics-based learning models are shown in Table \ref{Task 3: Physics-based learning model matrices (Lagrangian learning model)} and \ref{Task 3:Physics-based learning model matrices (Hamiltonian learning model)}, where ${G}({q})$ represents the potential forces, i.e., $\frac{\partial V({q})}{\partial {q}}$. As Table \ref{Task 3:Physics-based learning model matrices (Hamiltonian learning model)} shows, HNNs can learn the physically meaningful matrices, while LNNs only learn one of the solutions satisfying the Euler-Lagrangian equation. Comparing the corresponding matrices in Table \ref{Task 3: Mathematical model matrices} and \ref{Task 3: Physics-based learning model matrices (Lagrangian learning model)}, we can find that the matrices and vectors learned by the LNNs are related to the real parameters through a transformation ${P}({q})$.

\subsection{Two-segment 3D soft manipulator}
The two-segment soft manipulator model is simulated in MATLAB, where the configuration space is also defined as in the one-segment case. The training and testing information for this task is shown in Table \ref{table: two-segment simulate}. Figure \ref{fig: two_segment_LNN} summarizes the prediction results of the $50Hz$, $100Hz$, and $1000Hz$ learned model. From the simulations, we conclude that the higher the sampling frequency within a certain range, the more accurate the learned model is.
\begin{table}[!htb]
\caption{Two-segment simulated soft manipulator training and testing detailed information}
\centering
\begin{tabular}{ccccc} 
\hline
                 &\multirow{2}{*}{\makecell[c]{Black-box model \\ 100Hz}} & \multicolumn{3}{c}{Lagrangian-based learned model}                                                                                                            \\
                 & \multicolumn{1}{c}{}                                         & 50Hz                                               & 100Hz                                              & 1000Hz                                              \\ 
\hline
model(width $\times$ depth)            & $152 \! \times \!  3$                                               & $42 \! \times \!  3,   5 \! \times \!  3, 42 \! \times \!  2$ & $42 \! \times \!  3,   5 \! \times \!  3, 42 \! \times \!  2$ & $42 \! \times \!  3,   5 \! \times \!  3, 42 \! \times \!  2$  \\
sample number    & 59200                                    &45000                         & 45000                         & 45000                         \\
training epoch   & 15000                                    & 5500                          & 5500                           & 5500                           \\
traning error    & $3.536e^{-4} \pm 1.08e^{-3}$                                           & $5.916e^{-4} \pm 8.61e^{-3}$                                & $1.652e^{-4} \pm 2.12e^{-2} $                                & $1.822e^{-7} \pm 6.67e^{-6}$                                  \\
prediction error $[m]$ & $44.683\pm 4.518$(10s)                                           & $2.098 \pm 1.253$(10s)                                 & $1.690 \pm 0.673$(10s)                                 & $0.089 \pm 0.278$(10s)                                  \\
\hline
\end{tabular}
\label{table: two-segment simulate}
\end{table}
\begin{figure}[!htb!]
    \centering
    \includegraphics[width=1.0\columnwidth]{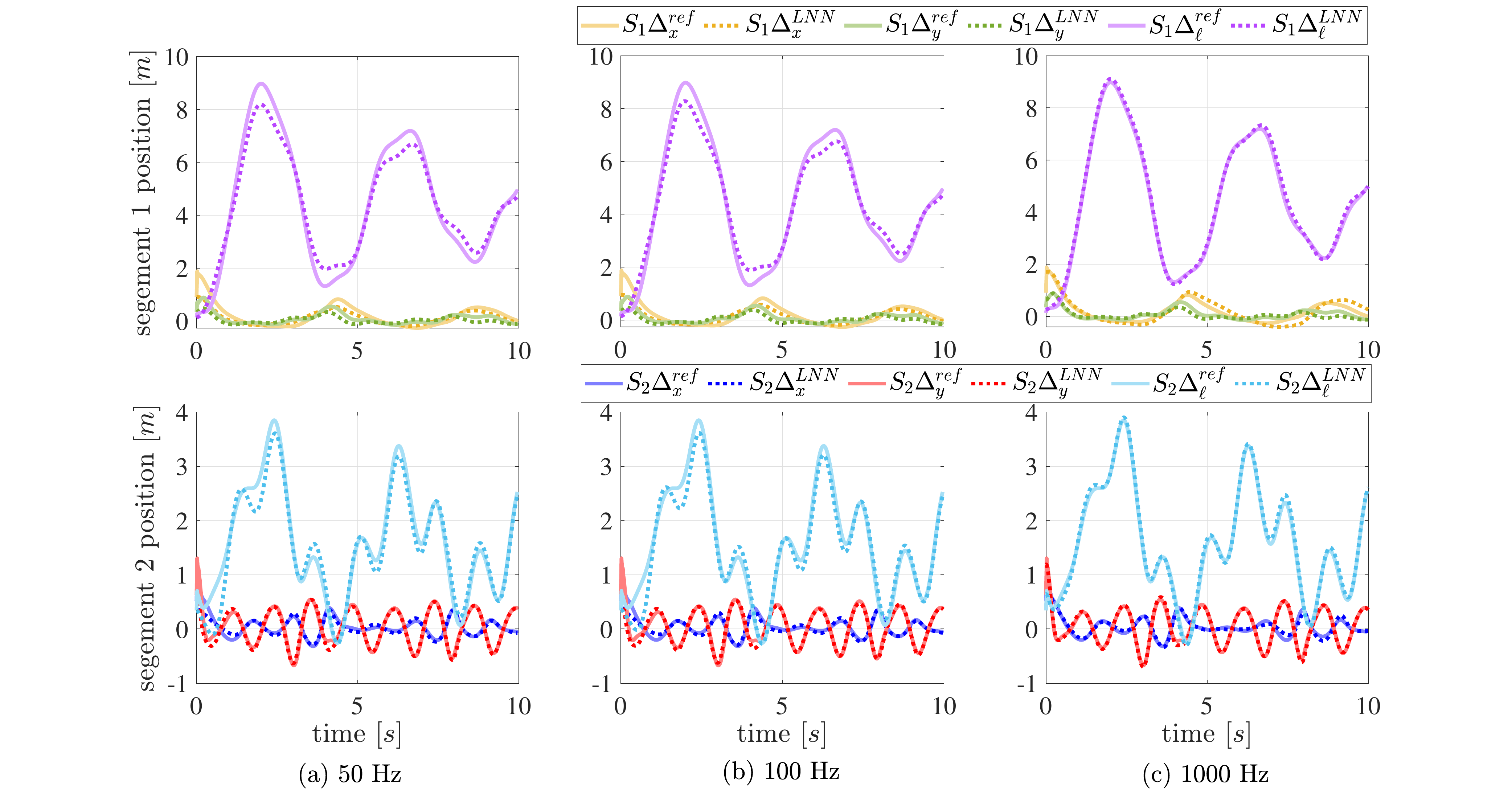}
    \caption{Two-segment soft manipulator prediction performances under different sampling frequencies}
    \label{fig: two_segment_LNN}
\end{figure}

\begin{figure}[!h]
\centering
\subfigure{
\includegraphics[width=7.2cm]{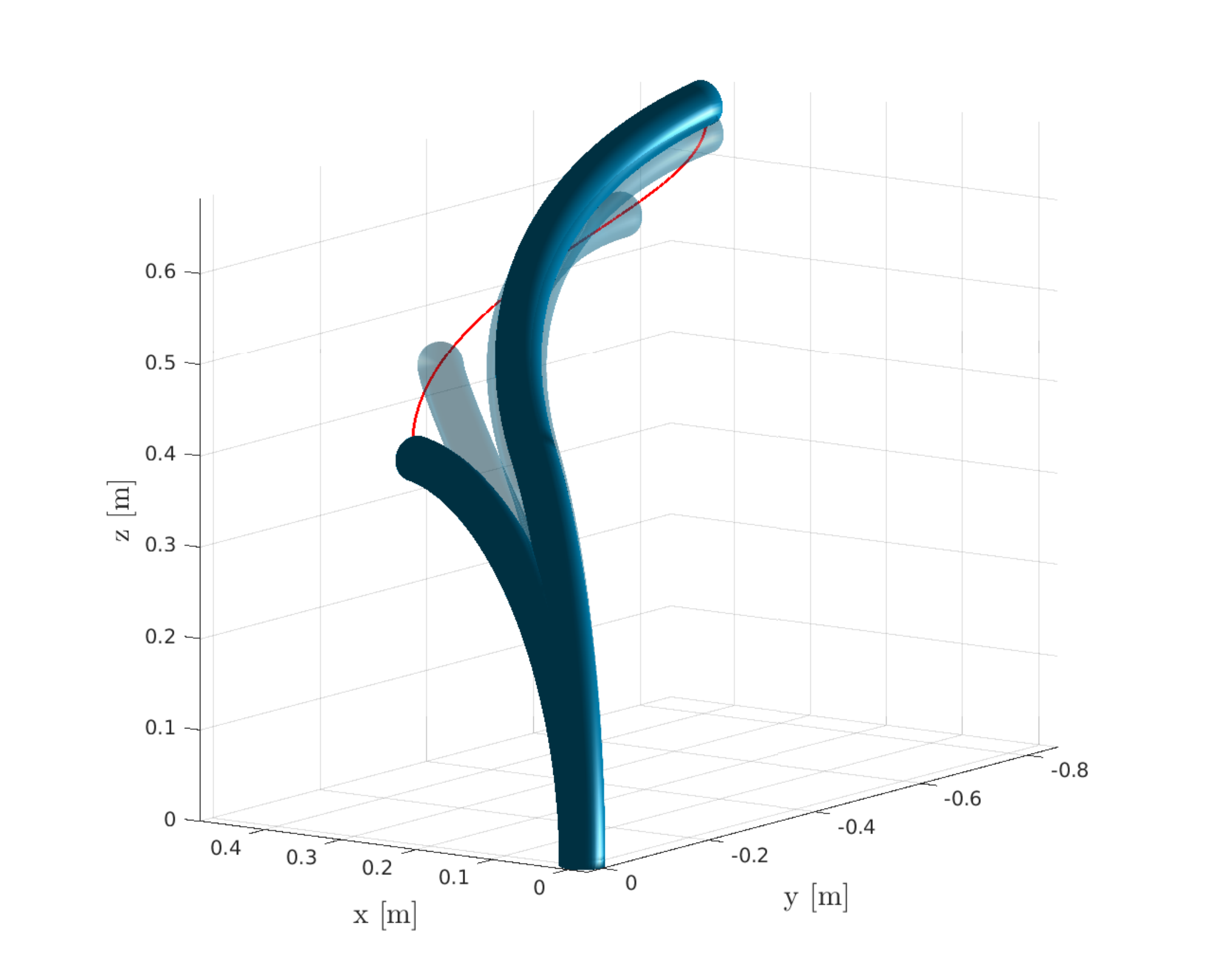}
\label{q0}
}
\quad
\subfigure{
\includegraphics[width=7.2cm]{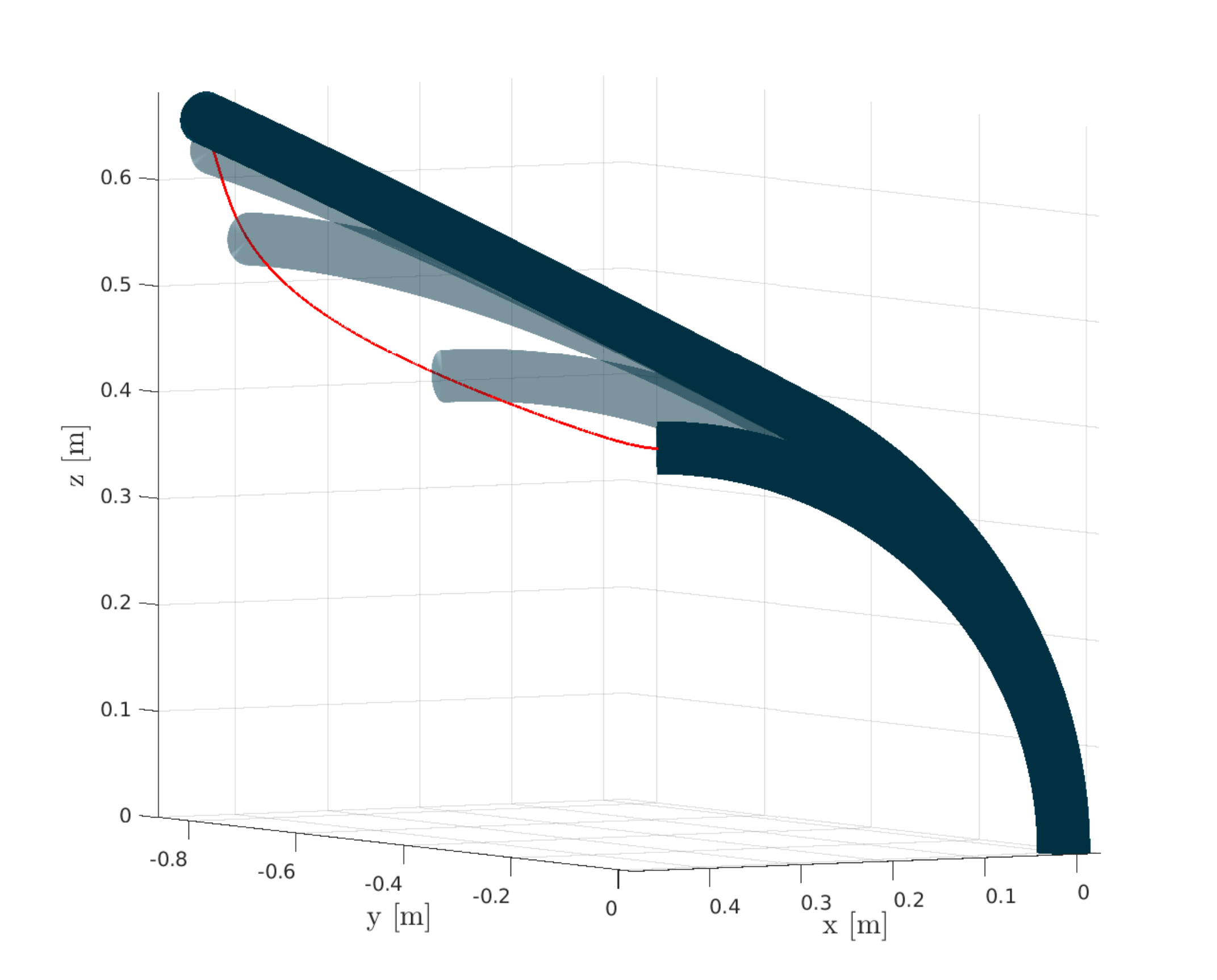}
\label{q1}
}
% \quad
% \subfigure{
% \includegraphics[width=4.5cm]{figures/two_segment/q2.eps}
% \label{q2}
% }
% \quad
% \subfigure{
% \includegraphics[width=4.5cm]{figures/two_segment/q3.eps}
% \label{q3}
% }
% \quad
% \subfigure{
% \includegraphics[width=4.5cm]{figures/two_segment/q4.eps}
% \label{q4}
% }
% \centering
% \includegraphics[width=0.5\columnwidth]{figures/two_segment/control12.eps}
\caption{The sequence of movements at the times 0.0s, 0.1s, 0.3s, 0.6s, and 1.0s executed by the two-segment soft robot as a result of the implementation of the LNN-model-based controller. The red line represents the tip's position }\label{fig: two_sgement_sequence}
\end{figure}

Based on the learned model trained at 1000Hz, we devise a PINN-based control loop as in \eqref{eq:controller_PD}. To demonstrate the performance of the designed controller, we employ it to control the two-segment soft manipulator in MATLAB. The proportional gains ${K}_{\tt P}$ and derivative gains ${K}_{\tt D}$ are set to 10 and 50, respectively, for all six configurations. The alterations in the states of the two-segment manipulator under control are depicted in Figure \ref{fig: two_sgement_sequence}, whereas the performance of the controller is demonstrated in Figure \ref{fig: simulated_two_control_example1}. Results indicate that the controller is capable of tracking a static setpoint within one second while keeping the root mean square error (RMSE) less than 0.23\%, and exhibits a stable and minimal overshoot performance. These performances underscore the reliability and efficiency of the designed controller based on the learned model.

\begin{figure}[!htb]
    \centering
    \includegraphics[width=1.0\columnwidth]{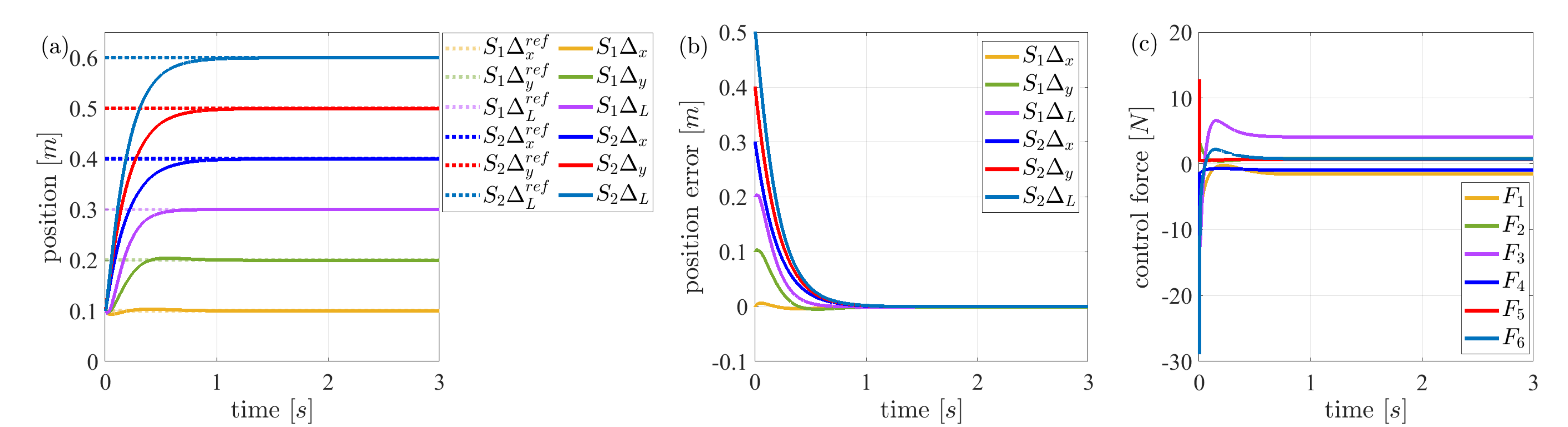}
    \caption{Two-segment soft manipulator model-based controller performance: (a) shows the evolution of the configuration variables and the desired state with dotted lines; (b) shows the error between the desired states and current states; (c) shows control effort.}
    \label{fig: simulated_two_control_example1}
\end{figure}

\subsection{Panda robot}
Table \ref{table: franka-simulation} presents the training and testing data of the simulated Panda in PyBullet, while Figure \ref{fig: franka_simulation} displays the prediction results obtained from the learned model. The model exhibits relatively accurate prediction performance within 1 second (i.e. continuous prediction for 1000 steps). Furthermore, the Lagrangian-based models can achieve long-term forecasting by updating the input values of the learned model to the real states at a fixed rate, typically ranging from 50 to 100 Hz.
\begin{table}[!h]
\centering
\caption{Franka simulation detailed information (1000Hz)}
\begin{tabular}{ccc} 
\hline
                        & Black-box model            & Lagrangian-based learned model          \\ 
\hline
model (width $\times$ depth)                   & $120 \! \times \! 4$       & $40 \! \times \! 3, 20 \! \times \! 2$  \\
sample number           & 550000 & 25000               \\
training epoch          & 10000  & 10000               \\
traning error           & $1.476e^{-4} \pm 2.69e^{-3}$       & $1.424e^{-4} \pm 2.90e^{-3} $                   \\
prediction error/ $[rad]$ & $5.132 \pm 15.691$(2s)         & $98.6937 \pm 6.411$(2s)                     \\
\hline
\end{tabular}
\label{table: franka-simulation}
\end{table}

\begin{figure}[!htb]
    \centering
    \includegraphics[width=0.7\columnwidth]{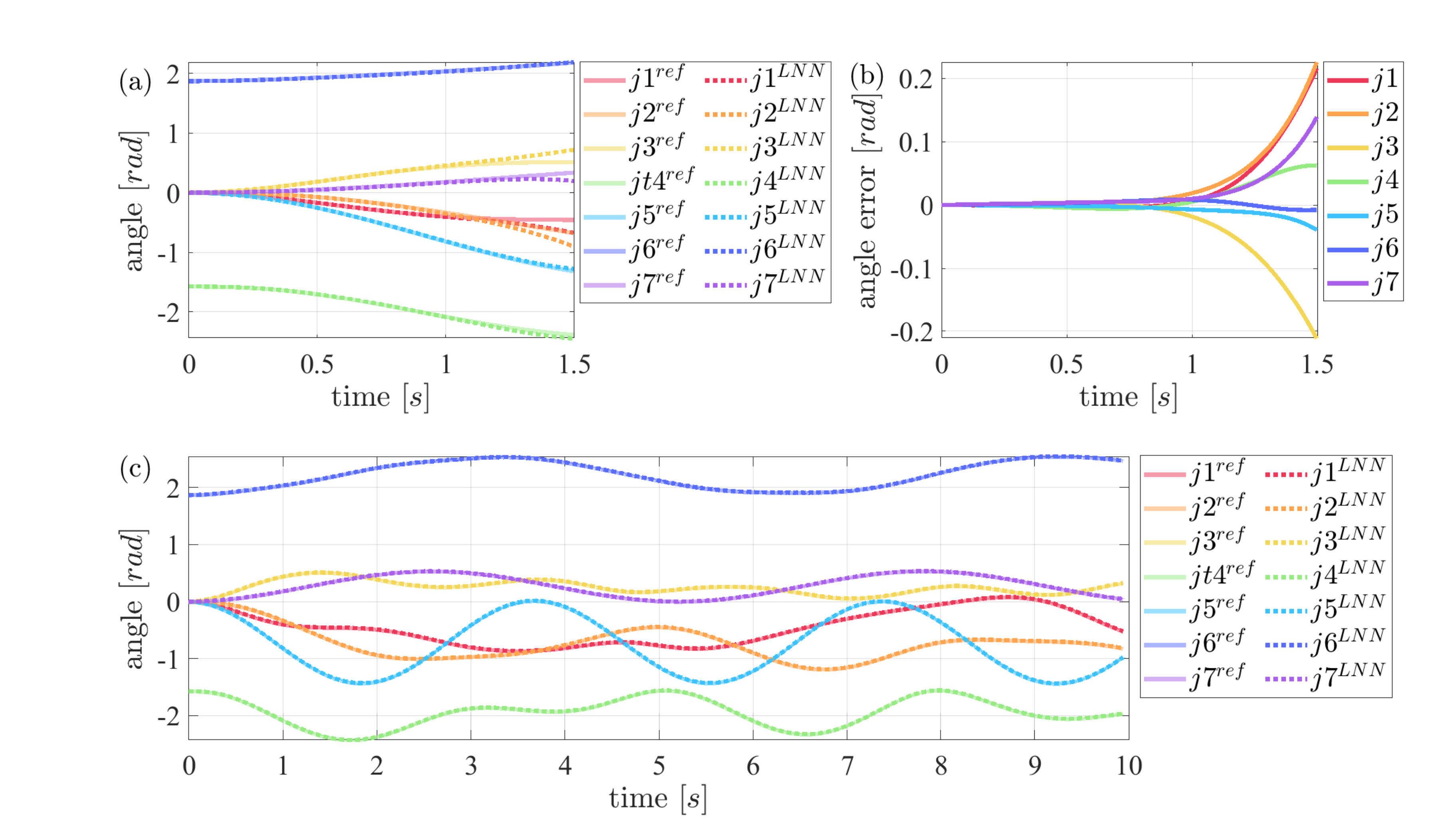}
    \caption{Franka Emika Panda learned model prediction results: (a) shows 1500 steps prediction in a row; (b) is the angle errors of the prediction with respect to the ground truth; (c) shows the long prediction results with 50-step window size.}
    \label{fig: franka_simulation}
\end{figure}
 
\begin{figure}[!htb]
    \centering
    \includegraphics[width=0.8\columnwidth]{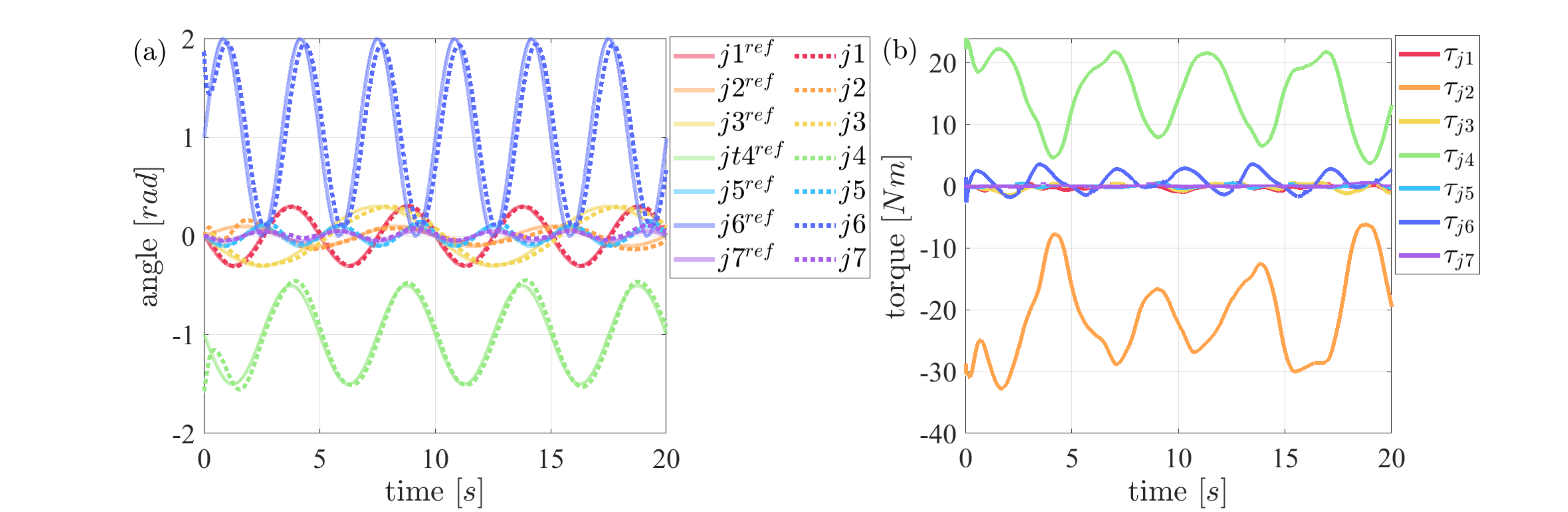}
    \caption{Performance of the model-based controller designed using the model learned by the LNNs. The desired trajectories are plotted with dotted lines.}
    \label{fig: franka_simulation_control}
\end{figure}
Based on this learned model, we build the tracking controller discussed in Sec. \ref{sec:mb_control}.
The results are depicted in Figure \ref{fig: franka_simulation_control}, where we observe that our controller has a fast response time and can quickly adapt to changes in the reference signal. It can maintain high accuracy and low phase lag, which makes it well-suited for tracking fast-changing signals.

% Modified PINNs can learn the dynamic model of soft manipulators by considering the PCC simplification and the 7-DoFs rigid robots. All results in this section show that learning becomes more instructive and directional with physics priors. Physics-based learning models trained with fewer data are more general and robust. Therefore, continuous long-term and variable step-size predictions can be achieved. Furthermore, the learned model enables decent anticipatory control, and a naive PD can be integrated for higher performance. \com{This paragraph seems odd here. I would modify it a bit and include it in the conclusions. If you want to include it here, you could create a new subsection named "Discussion". Please see what I wrote in Conclusions.}

\section{Experimental Validation} \label{sec: Experimental Validation}
\subsection{One-segment tendon-driven soft manipulator -- NECK}
We validate the proposed approach in the platform depicted in Figure \ref{fig:real_experiment}, which is constructed based on \cite{9762144, GithubExperimentProject}. We consider two different data preprocessing methods. (i) Moving average method: This method reduced the noise and outliers in the data, generating a more stable representation of underlying trends. However, it may overlook intricate relationships between variables, resulting in some information loss. (ii) Polynomial fitting: This method captured non-linear patterns in the data. However, it was susceptible to the influence of outliers, resulting in spurious information that may compromise the quality of the trained model.

The training and testing information is shown in Table \ref{table: neck}.
\begin{figure}[!htb]
    \centering
    \includegraphics[width=0.6\columnwidth]{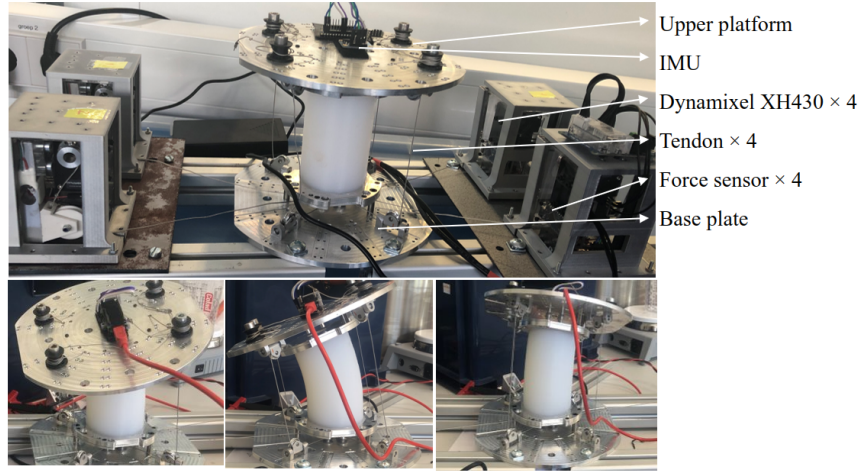}
    \caption{Experiment platform: One-segment tendon-driven soft manipulator equipped with IMU}
    \label{fig:real_experiment}
\end{figure}

The method of moving average is implemented in MATLAB through the utilization of the \texttt{movmean} function, with a prescribed window size of 50 points. The processed data are used for training the LNNs. In Figure \ref{fig:smoothing_data_prediction_no_update}, we compare the continuous prediction ability of black-box and Lagrangian-based learning models. The prediction performance in this figure indicates that the Lagrangian-based learning model exhibits superior predictive accuracy in this sample. Furthermore, Figure \ref{fig:smoothing_data_prediction_no_update} (c) shows that the learning model can realize long-term predictions under the short-term update.
\begin{table}[ht]
\centering
\caption{The tendon-driven soft robot -- NECK training and testing information}
\begin{tabular}{cccc} 
\hline
\multicolumn{1}{l}{}       &                  & Black-box model           & Lagrangian-based learned model                             \\ 
\hline
\multirow{5}{*}{smoothing} & model            & $60 \! \times \! 3$       & $21 \! \times \! 2, 25 \! \times \! 2, 10 \! \times \! 2$  \\
                           & sample number    & 69426 & 69426                                  \\
                           & training epoch   & 10000 & 3000                                   \\
                           & traning error    & $1.985e^{-2}\pm 1.85e^{-1}$      & $2.277e^{-2} \pm 2.39e^{-1}$                                       \\
                           & prediction error $[^\circ]$ & $13.229 \pm 60.762$ (5s)       & $2.429 \pm 1.259$ (5s)                                       \\ 
\hline
\multirow{5}{*}{fitting}   & model            & 6$0 \! \times \! 3$       & $21 \! \times \! 2, 25 \! \times \! 2, 10 \! \times \! 2$  \\
                           & sample number    & 57950 & 48200                                  \\
                           & training epoch   & 5000  & 5000                                   \\
                           & training error    & $4.431e^{-3} \pm 3.07e^{-2}$      & $2.758e^{-3} \pm 2.84e^{-2} $                                      \\
                           & prediction error$[^\circ]$ & $8.368 \pm 12.575$ (5s)         & $6.426 \pm 36.237$ (5s)                                           \\
\hline
\end{tabular}
\label{table: neck}
\end{table}

\begin{figure}[!htb]
    \centering
    \includegraphics[width=0.6\columnwidth]{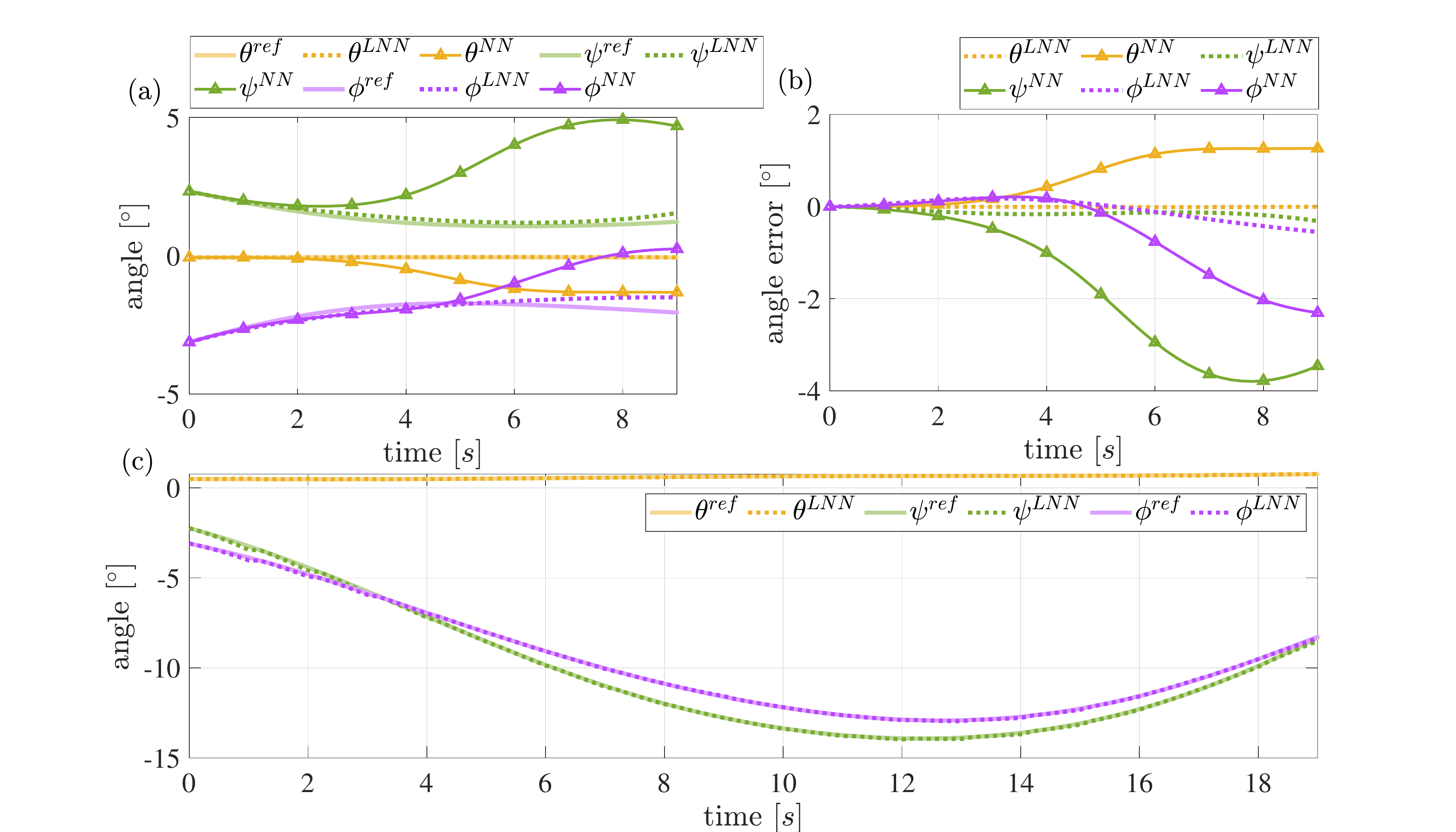}
    \caption{The smoothing data black-box model ($\bigtriangleup$) and physics-based learning model (- -) continuous prediction results: (a) and (b) show prediction 43 prediction steps in a row; (c) depicts the prediction results with 5-step window size.}
    \label{fig:smoothing_data_prediction_no_update}
\end{figure}
\begin{figure}[!htb]
    \centering
    \includegraphics[width=0.6\columnwidth]{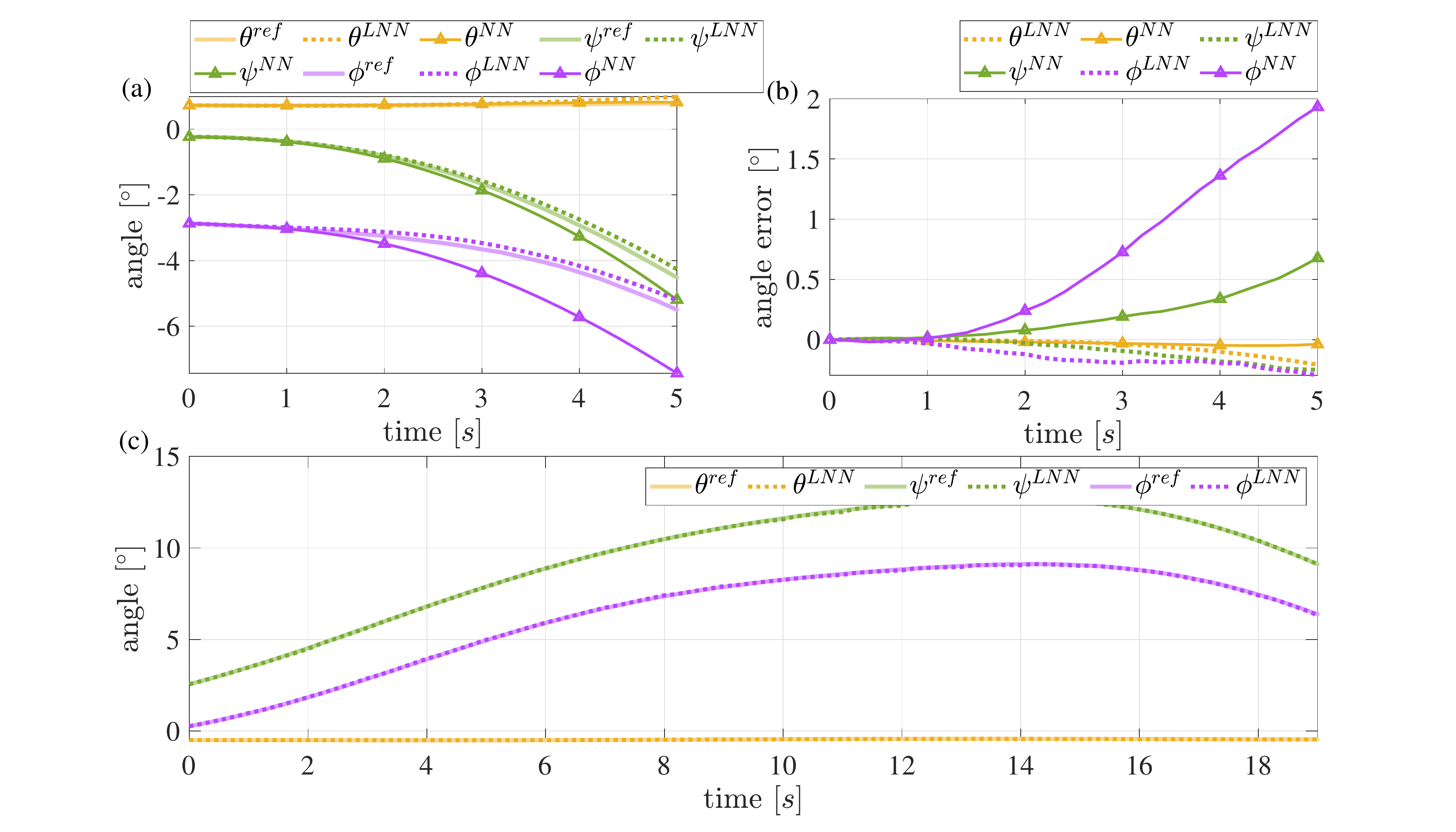}
    \caption{The fitting data black-box model ($\bigtriangleup$) and physics-based learning model ($\cdots$) continuous prediction results: (a) and (b) show 25 prediction steps in a row; (c) shows the prediction results with 5-step window size.}
    \label{fig:fitting_data_prediction_no_update}
\end{figure}
The polynomial fitting of the data is done in MATLAB using the function \texttt{polyfit}. The prediction results of the model are shown in Figure \ref{fig:fitting_data_prediction_no_update}.  The learned model exhibits a decent performance when the window size is reduced, as shown in Figure \ref{fig:fitting_data_prediction_no_update}(c). In contrast to the previous model, this model exhibits significant prediction errors shown in Table \ref{table: neck}. This can be caused by the significant noise in the sensors and misinformation caused by the approximation used to fit the data.

\subsection{Rigid Robot -- Franka Emika Panda}
The collected data are processed through a Butterworth filter in MATLAB to reduce noise. Further details are provided in Table \ref{table: franka-real}. In the  experiment, we observe small joint acceleration, which results in minimal velocity change. To prevent the network from focusing solely on learning a large mass matrix and neglecting other important factors, we utilize a scaling sigmoid function. This function ensures that the elements in the mass matrix are scaled within a specific range. For this particular case, we have set the scaling factor to 3.50.
\begin{table}[!htb]
\centering
\caption{Franka experiment detailed information (500Hz)}
\begin{tabular}{ccc} 
\hline
                        & Black-box model            & Lagrangian-based learned model          \\ 
\hline
model (width $\times$ depth)                  & $120 \! \times \! 5$       & $40 \! \times \! 3, 20 \! \times \! 2$  \\
sample number           & 550000 & 25000              \\
training epoch          & 10000 & 3000             \\
traning error           & $1.371e^{-5} \pm 2.03e^{-5}$       & $1.68e^{-7} \pm 6.64e^{-6}$                    \\
prediction error$[rad]$ & $182.495 \pm 64.645$ (2s)       & $2.681 \pm 1.383$ (2s)                      \\
\hline
\end{tabular}
\label{table: franka-real}
\end{table}

Figure \ref{fig: franka_real} illustrates the predictive performance of our physics-based model, where Figure \ref{fig: franka_real} (b) depicts the continuous prediction error within 2 seconds or 1000 prediction steps and (c) shows that updating the model's input with real-time state data can help us make a long prediction. 

\begin{figure}[h]
    \centering
    \includegraphics[width=1.0\columnwidth]{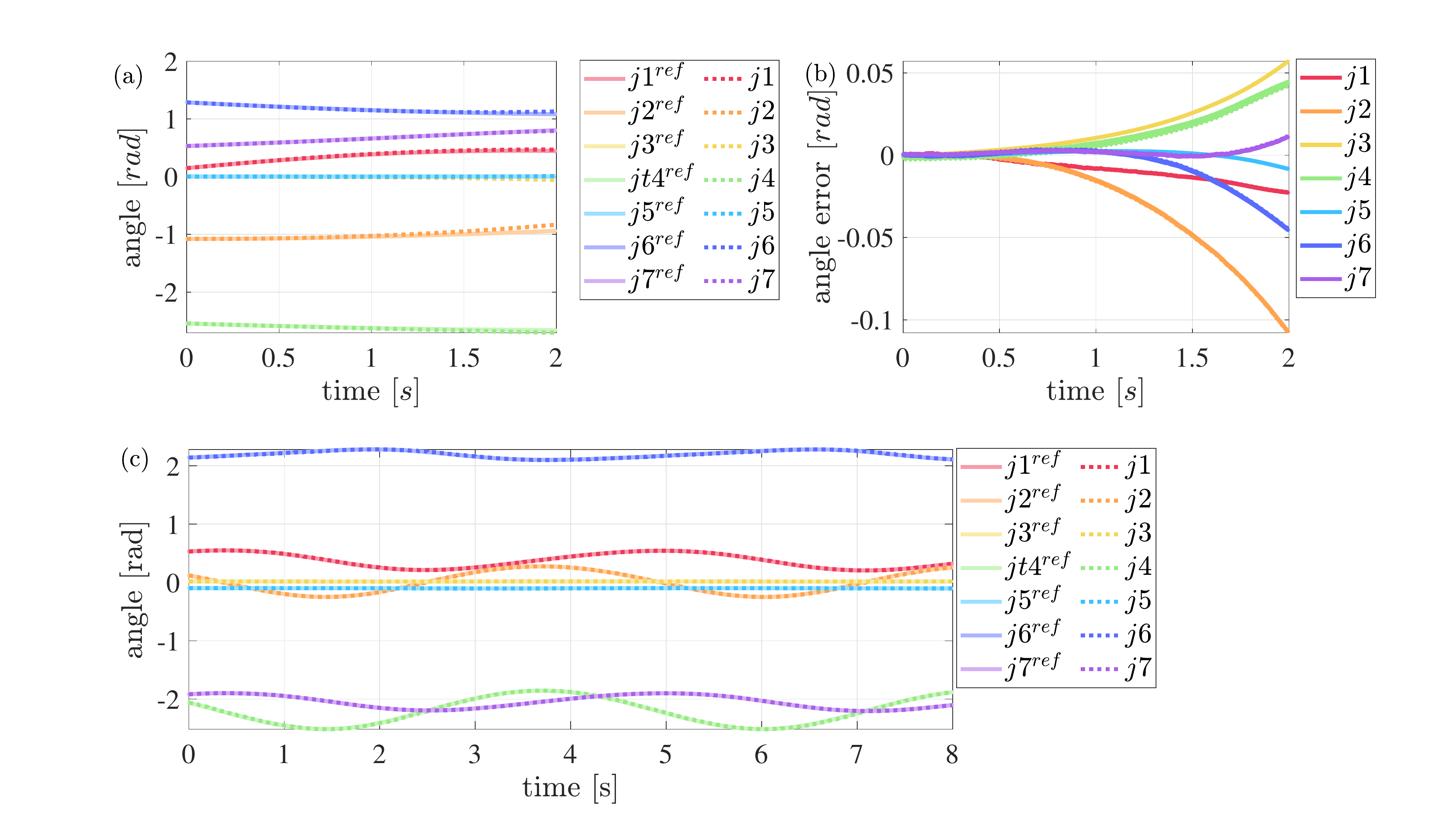}
    \caption{Panda physics-based learning model prediction results: (a) and (b) show prediction of about 800 steps in a row; (c) depicts the prediction results with 5-step window size.}
    \label{fig: franka_real}
\end{figure}
A controller based on the equation presented in \eqref{eq: controller} is proposed for the actual robot. The proportional gain matrix, ${K}_{\tt P}$, is set to a diagonal matrix with entries $600$, $600$, $600$, $600$, $250$, $150$, and $50$, respectively. The derivative gain matrix, ${K}_{\tt D}$, is set to a diagonal matrix with entries $30$, $30$, $30$, $30$, $10$, $10$, and $5$, respectively. Figure \ref{fig: franka_movement} illustrates a series of photographs depicting the periodic movement used to track a sinusoidal trajectory within a time frame of 10 seconds. The whole tracking performance is shown in Figure \ref{fig: franka_real_control}. 
\begin{figure}[h]
    \centering
    \includegraphics[width=0.9\columnwidth]{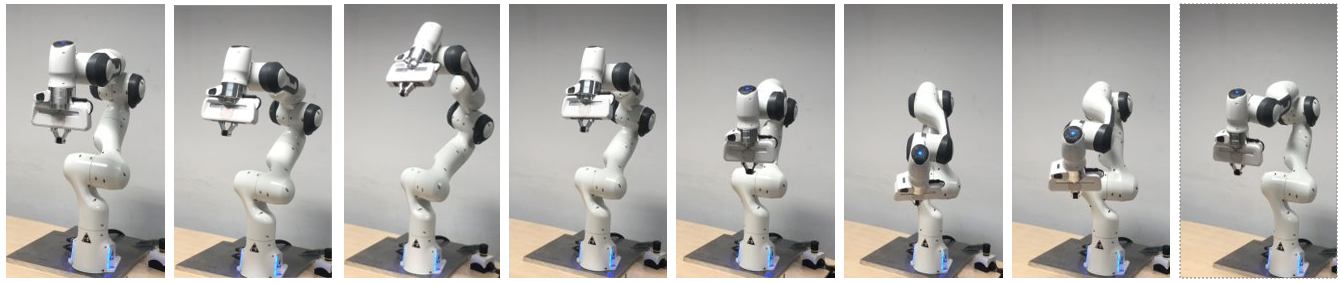}
    \caption{Photo sequence of one periodic movement resulting from the application of the LNN-model-based controller tracking trajectory}
    \label{fig: franka_movement}
\end{figure}

\begin{figure}[h]
    \centering
    \includegraphics[width=0.7\columnwidth]{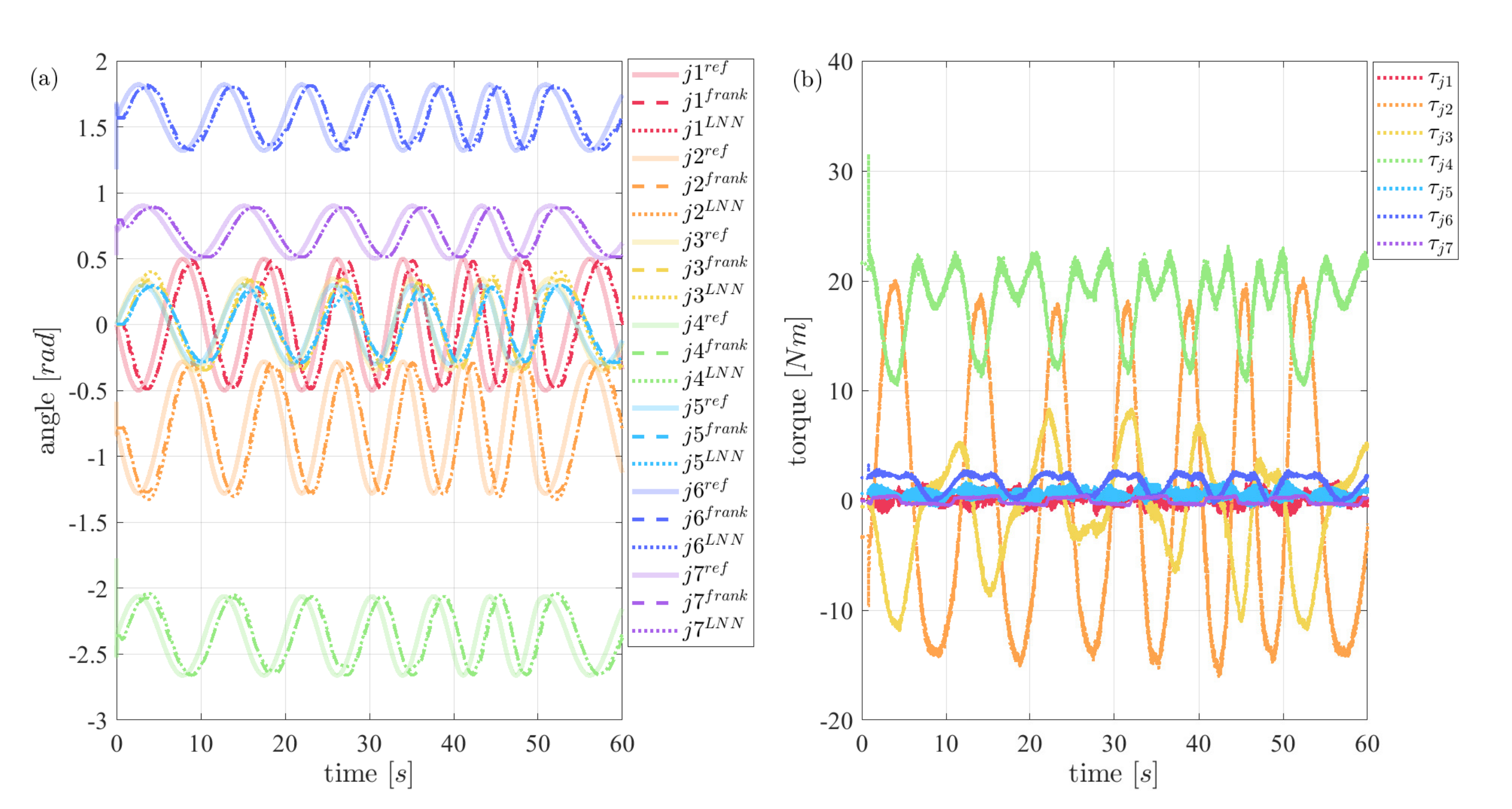}
    \caption{Performance of the model-based controller that is designed using the learned model.}
    \label{fig: franka_real_control}
\end{figure}

\begin{figure}[h]
    \centering
    \includegraphics[width=0.9\columnwidth]{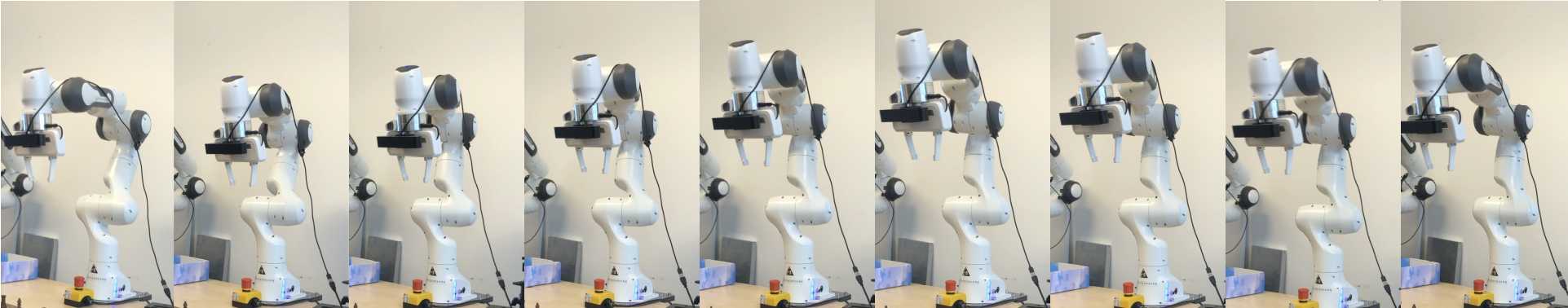}
    \caption{Photo sequence of helical motion of the end-effector by using LNN-model-based controller}
    \label{fig: sequence4helixMotion}
\end{figure}

\begin{figure}[h]
    \centering
    \includegraphics[width=0.8\columnwidth]{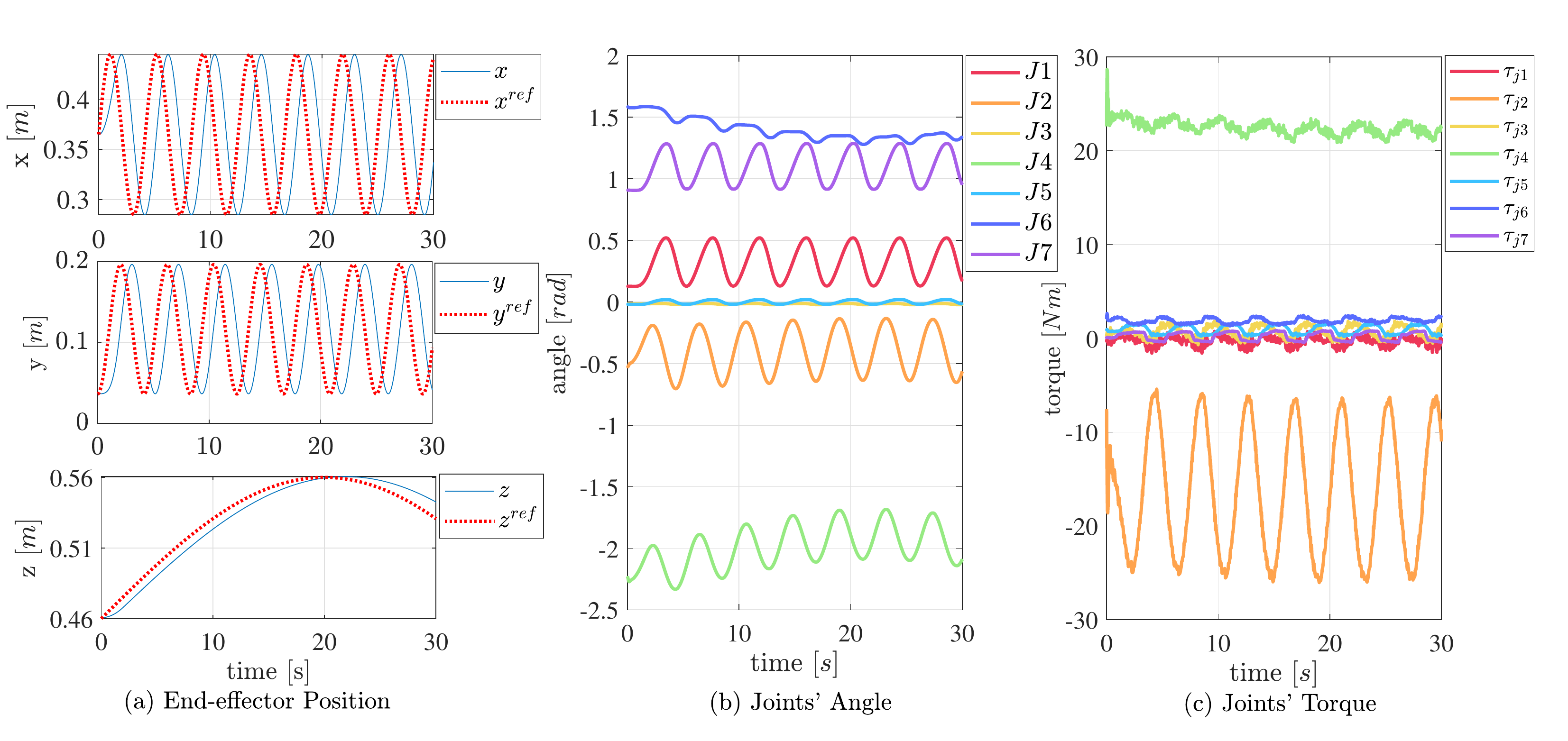}
    \caption{Performance of the model-based controller that is designed using the learned model: (a) shows the desired end-effector trajectory; (b) shows the corresponding joints' angle and the control results; (c) is the controller's input torques for such motion.}
    \label{fig: franka_circle}
\end{figure}
Furthermore, we have presented the trajectory of the end-effector, which is a helical motion shown in Figure \ref{fig: sequence4helixMotion}, and its resultant control effect has been visually demonstrated in Figure \ref{fig: franka_real_control}.

In these Figures, we can observe that the designed controller has satisfactory performance, as evidenced by its ability to track a desired trajectory. The tracking error, while present in some joints, remains within acceptable bounds and does not significantly impair the overall performance of the controller in practical applications. An examination of the controller's performance reveals that, while generally effective, its performance exhibits some degree of variability across different joints. The overall performance of the controller remains within acceptable levels and suggests its potential for effective use in real-world applications.

\section{Conclusions}
\label{sec:6_Conclusion}
This paper presented an approach to consider damping and the interaction between robots an/d actuators in PINNs---specifically, LNNs and HNNs---, improving the applicability of these neural networks for learning dynamic models. Moreover, we used the Runge-Kutta4 method to avoid acceleration measurements, which are often unavailable. The modified PINNs proved suitable for learning the dynamic model of rigid and soft manipulators. For the latter, we considered the PCC approximation to obtain a simplified model of the system. 

The modified PINNs approach exploits the knowledge of the underlying physics of the system, which results in a largely improved accuracy in the learned models compared with the baseline models, which were trained using an fully connected network. The results show that PINNs exhibit a more instructive and directional learning process because of the prior knowledge embedded into the approach. Notably, physics-based learning models trained with fewer data are more general and robust than the traditional black-box ones. Therefore, continuous long-term and variable step-size predictions can be achieved. Furthermore, the learned model enables decent anticipatory control, where a naive PD can be integrated for a good performance, as illustrated in the experiments performed with the Panda robot.

% \medskip
% \textbf{Supporting Information} \par %Please delete the Suppporting Information statement if it is not applicable. Please supply Supporting Information in another file. Supporting information should not be provided in .tex format
% Supporting Information is available from the Wiley Online Library or from the author.
% Acknowledgements
\medskip
\textbf{Acknowledgements} \par %delete if not applicable))
We wish to acknowledge the EMERGE for their financial support, which enabled us to carry out this research. We are also grateful to Bastian Deutschmann, the inventor of the NECK experimental platform, which greatly facilitated our work. I would also like to express my deepest gratitude to Francesco Stella and  Tomás Coleman for their invaluable guidance and help in the experiments. Finally, we extend our appreciation to  our colleagues for their insightful feedback and constructive criticism, which helped refine our ideas and methods. 
% References
\medskip

\end{document}